\documentclass[review,onefignum,onetabnum,final]{siamonline250211}

\usepackage{lipsum}
\usepackage{amsfonts}
\usepackage{bm}
\usepackage{mathtools}
\usepackage{graphicx}
\usepackage{wrapfig}
\usepackage{booktabs}
\usepackage{epstopdf}
\usepackage{algorithm}
\usepackage{algorithmic}
\ifpdf
  \DeclareGraphicsExtensions{.eps,.pdf,.png,.jpg}
\else
  \DeclareGraphicsExtensions{.eps}
\fi

\usepackage{enumitem}
\setlist[enumerate]{leftmargin=.5in}
\setlist[itemize]{leftmargin=.5in}

\newsiamremark{remark}{Remark}
\newsiamremark{hypothesis}{Hypothesis}
\crefname{hypothesis}{Hypothesis}{Hypotheses}
\newsiamthm{claim}{Claim}

\headers{Mirror Bridges Between Probability Measures}{L. Mattos Da Silva, S. Sell\'an, F. Vargas, and J. Solomon}

\title{Mirror Bridges Between Probability Measures
}

\author{Leticia Mattos Da Silva\thanks{Massachusetts Institute of Technology, Cambridge, MA.}
\and Silvia Sella\'n\thanks{Columbia University, New York, NY.}
\and Francisco Vargas\thanks{Xaira Theurapetics, London, UK. \funding{The first author acknowledges the generous support of a MathWorks Engineering Fellowship. The last author acknowledges the generous support of Army Research Office grants W911NF2010168 and W911NF2110293, of National Science Foundation grant IIS2335492, from the CSAIL Future of Data program, from the MIT–IBM Watson AI Laboratory, from the Wistron Corporation, and from the Toyota–CSAIL Joint Research Center.}}
\and Justin Solomon\footnotemark[1]
}

\usepackage{amsopn}

\makeatletter
\newcommand*{\addFileDependency}[1]{
  \typeout{(#1)}
  \@addtofilelist{#1}
  \IfFileExists{#1}{}{\typeout{No file #1.}}
}
\makeatother

\newcommand*{\myexternaldocument}[1]{%
    \externaldocument{#1}%
    \addFileDependency{#1.tex}%
    \addFileDependency{#1.aux}%
}

\newcommand{\ampalgo}{
\FOR{$k \in \{0, \dots, K-1\}$}
\WHILE{not converged}
  \STATE Sample $\rmX_0^j\sim\pi$ and $\sigma^j\in\R$ from $\left[\sigma_{\rm{min}},
  \sigma_{\rm{max}}\right]$ for $j \in \{0,\dots,M-1\}$.
  \STATE Compute trajectories $\{\rmX_{i}^j\}_{i,j = 0}^{M-1,N-1}$ via (\ref{eq:eulermaruyama}) using $f(x) = v_t^{\theta^{k}}(x)$.
  \STATE Do gradient step on $\theta^{k+1/2}$ using (\ref{eq:loss}).
  \STATE Let $v_t^{\theta_{k+1}} = \frac{1}{2}(v_t^{\theta_k} + v_t^{\theta_{k+1/2}})$.
  \ENDWHILE
  \ENDFOR
  \STATE \textbf{Output:} $v_t^{\theta^{\star}}$
}

\def\eqref#1{equation~\ref{#1}}

\def\1{\bm{1}}

\def\rmI{{\mathbf{I}}}

\def\rmW{{\mathbf{W}}}
\def\rmX{{\mathbf{X}}}
\def\rmY{{\mathbf{Y}}}
\def\rmZ{{\mathbf{Z}}}

\def\mI{{\bm{I}}}

\DeclareMathAlphabet{\mathsfit}{\encodingdefault}{\sfdefault}{m}{sl}
\SetMathAlphabet{\mathsfit}{bold}{\encodingdefault}{\sfdefault}{bx}{n}

\def\gA{{\mathcal{A}}}
\def\gB{{\mathcal{B}}}

\def\gJ{{\mathcal{J}}}

\def\gN{{\mathcal{N}}}

\def\sD{{\mathbb{D}}}

\def\sP{{\mathbb{P}}}
\def\sQ{{\mathbb{Q}}}

\def\sS{{\mathbb{S}}}

\newcommand{\pdata}{p_{\rm{data}}}
\newcommand{\pprior}{p_{\rm{prior}}}

\newcommand{\rmd}{{\rm{d}}}

\newcommand{\E}{\mathbb{E}}

\newcommand{\R}{\mathbb{R}}

\newcommand{\KL}{D_{\mathrm{KL}}}

\DeclareMathOperator*{\argmin}{arg\,min}

\ifpdf
\hypersetup{
  pdftitle={Mirror Bridges Between Probability Measures},
  pdfauthor={L. Mattos Da Silva, S. Sell\'an, F. Vargas, and J. Solomon}
}
\fi

\myexternaldocument{ex_supplement}

\begin{document}

\maketitle

\begin{abstract}
  Resampling from a target measure whose density is unknown is a fundamental problem in mathematical statistics and machine learning. A setting that dominates the machine learning literature consists of learning a map from an easy-to-sample prior, such as the Gaussian distribution, to a target measure. Under this model, samples from the prior are pushed forward to generate a new sample on the target measure, which is often difficult to sample from directly.   A related problem of particular interest is that of generating a new sample proximate to or otherwise conditioned on a given input sample. In this paper, we propose a new model called the \emph{mirror bridge} to solve this problem of conditional resampling. Our key observation is that solving the Schr\"odinger bridge problem between a distribution and itself provides a natural way to produce new samples, giving in-distribution variations of an input data point. We demonstrate how to efficiently estimate the solution of this largely overlooked version of the Schr\"odinger bridge problem. 
  We show that our proposed method leads to significant algorithmic simplifications over existing alternatives, in addition to providing control over in-distribution variation. Empirically, we demonstrate how these benefits can be leveraged to produce proximal samples in a number of application domains.
\end{abstract}



\section{Introduction}
\label{sec:introduction}

Mapping one probability distribution to another is a central technique in mathematical statistics and machine learning. Myriad computational tools have been proposed for this critical yet often challenging task. Models and techniques for optimal transport provide one class of examples, where methods like the Hungarian algorithm \cite{kuhn1955hungarian} map one distribution to another with optimal cost.  Adding entropic regularization to the static optimal transport problem yields efficient algorithms like Sinkhorn's method \cite{deming1940marginal,sinkhorn1964matrices}, which have been widely adopted in machine learning since their introduction by \cite{cuturi2013sinkhorn}. Static entropy-regularized optimal transportation is equivalent to a dynamical formulation known as the \emph{Schr\"odinger bridge problem} \cite{schrodinger1932bridge,leonard2014measures}, which has proven useful to efficiently compute an approximation of the optimal map paired with an interpolant between the input measures.

Inspired by these mathematical constructions and efficient optimization algorithms, several methods in machine learning rely on learning a map from one distribution to another. Beyond optimal transport, diffusion models, for instance, learn to reverse a diffusion process that maps data to a noisy prior. Special attention has been given to learning methods that accomplish this in a \emph{stochastic} manner, i.e., modeling the forward noising process using a stochastic differential equation (SDE). 

The most common learning applications of distribution mapping attempt to find a map from a simple prior distribution to a complex data distribution, either using a score-matching strategy \cite{song2019gradients, ho2020denoising, song2021scorebased} or leveraging a formulation of the Schr\"odinger bridge problem \cite{bortoli2021diffusion,shi2022conditional,shi2023diffusion,zhou2024denoising}; other learning applications map one complex data distribution to another \cite{cuturi2013sinkhorn,courty2017adaptation}. 

Rather than mapping a simple prior to a complex data distribution, in this paper we instead tackle the understudied problem of mapping a probability distribution to \emph{itself}, that is, finding a joint distribution whose marginals are both the same data distribution $\pi$. 
This task might seem inane at first glance, since two simple couplings satisfy our constraints: 
one is the independent coupling $p(x,y)=\pi(x)\pi(y)$, and the other is the ``diagonal'' map given by $p(x,y)=\pi(x)\delta_{y}$. The space of couplings between a measure and itself, however, is far richer than these two extremes and includes models whose conditional distributions are neither identical nor Dirac measures.

\begin{figure*}[t]
    \centering
    \includegraphics[width=\linewidth]{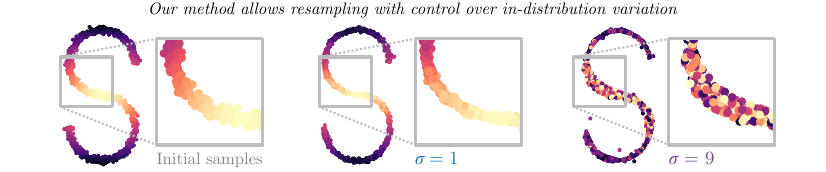}
    \vspace{-0.2in}
    \caption{We visualize the resampling of a 2D distribution obtained with our method for different values of noise $\sigma$. Higher noise value results in greater in-distribution sample variance.}
    \label{fig:teaser}
    \vspace{-0.1in}
\end{figure*}

We focus on the class of self-maps obtained by entropy-regularized transport from a measure to itself. Formally, we define a \emph{mirror Schr\"odinger bridge} to be the minimizer of the KL divergence $\KL ( \sP \;\Vert\; \sP^0 )$ over path measures $\sP$ with both initial and final marginal distributions equal to $\pi$, where $\sP^0$ is an Ornstein--Uhlenbeck process with noise $\sigma$. Mirror Schr\"odinger bridges are the stochastic counterpart to minimizing $\KL \left( p \;\Vert\; p^{0} \right)$, where $p^{0}$ is the probability density of the joint distribution associated with the path measure $\sP^0$, over the joint distributions $p$ on $\R^n \times \R^n$ satisfying the linear constraints $\int p(x,y) dy = \pi(x)$ and $\int p(x,y) dx =\pi (y)$. While the former minimizes the Kullback-Leibler divergence on path space, the latter is a minimization over density couplings. 

Despite its simplicity, the mirror case of the Schr\"odinger bridge problem suggests a rich application space. Couplings with identical marginals have proven useful to enhance model accuracy in vision and natural language processing by reinterpreting attention matrices as transport plans \cite{sander2022sinkformers}. Few works, however, consider this task from the perspective of optimizing over path measures or provide control over the entropy of the matching at test time. \cite{albergo2023stochastic} propose a stochastic interpolant between a distribution and itself, but their interpolants are not minimal in the relative entropy sense. 
Such minimal interpolants are those with minimal kinectic energy, and in applications, minimizing the kinectic energy of a path has been correlated to faster sampling \cite{shaul2023kinectic}.

\subsection*{Contributions} We investigate the mirror Schr\"odinger 
bridge problem and demonstrate how it can be leveraged to obtain in-distribution variants of a given input sample. In particular, given a sample $x_0\sim p_\text{data}$, we build an estimate for the stochastic process $\{\rmX_t\}_{t \in [0,1]}$ with minimal relative entropy under which the sample $x_0$ arrives at some $x_1 \sim p_\text{data}$ with $x_1$ proximal but not identical to $x_0$. We call our estimate a \emph{mirror bridge}. 
Although this estimate might not be the mirror Schr\"odinger bridge, we prove that under mild conditions the former is a good approximation of the latter in an explicitly quantifiable way. 
Furthermore, we demonstrate that our algorithm for obtaining mirror bridges has computational advantages over alternatives used to compute Schr\"odinger 
bridges. 

Our contributions are twofold. First, on the theoretical side, we use the time symmetry of the mirror Schr\"odinger 
bridge to express it as the limit of iterates produced 
by 
a symmetrized version of the 
Iterative Proportional Fitting Procedure (IPFP). We show 
that this scheme admits a convenient first-order approximation that dramatically reduces computational expense in practice. Our mirror bridge is the limit of the latter procedure, and we provide an error bound to justify it as a good approximation.  
Second, in applications, the implementation of our method allows for sampling from the conditional distribution $\rmX_1 \mid \rmX_0 = x_0$, allowing us to 
control how proximal a generated sample $x_1$ is relative to the input sample $x_0$.

\section{Related Works}

In this section, we place our method in context of related work on solving the Schr\"odinger bridge problem.

\label{sec:related}
\paragraph*{Entropy regularized optimal transport} A few recent works employ the idea of a coupling with the same marginal constraints. The works \cite{feydy2019interpolating, mensch2019geometric} use \emph{static} entropy-regularized optimal transportation from a distribution to itself to build a cost function correlated to uncertainty. The work \cite{sander2022sinkformers} reinterprets attention matrices in transformers as transport plans from a distribution to itself, while \cite{agarwal2024iterated} analyzes this reinterpretation in the context of gradient flows. Also relevant is the work \cite{kurras2015symmetric}, which shows that, over discrete state spaces, Sinkhorn's algorithm can be simplified in the case of identical marginal constraints. In contrast, we study the same-marginal setting from the viewpoint of \emph{dynamical} Schr\"odinger bridges, motivated by the conditional resampling
problem introduced in \S\ref{sec:introduction}: given an input sample $x_0\sim \pi$, we wish to efficiently generate a new sample $x_1\sim \pi$ that remains controllably close to $x_0$. While a static coupling does define a conditional law of $x_1$ given $x_0$, turning that into an efficient sampler is nontrivial in the high-dimensional domains we target. This is precisely where the dynamical formulation shines, as it yields a diffusion process whose forward simulation \emph{is} the sampler:  starting from $x_0$, we simply run the learned bridge forward in time and output $x_1$. 

\paragraph*{Expectation maximization} Our method can be categorized under the umbrella of expectation maximization algorithms, drawing from the theory of information geometry. A number of recent papers introduce related formulations to machine learning; most relevant to us are the works of \cite{brekelmans2023expectation,vargas2023transport,vargas2024transport}. These works, however, focus on finding a path measure with distinct marginal constraints, overlooking the potential application to resampling and algorithmic simplifications obtained for the case in which the marginal constraints are the same.

\paragraph*{Schr\"odinger bridges and stochastic interpolants} Schr\"odinger bridges have been used to obtain generative models by flowing samples from a prior distribution to an empirical data distribution from which new data is to be sampled. Several methods have been proposed to this end: the works \cite{bortoli2021diffusion,vargas2021likelihood} iteratively estimate the drift of the SDE associated with the diffusion processes of half-bridge formulations. While the first uses neural networks and score matching, the latter employs Gaussian processes. From these, a number of extensions or alternatives have been presented; most relevant are \cite{peluchetti2023mixture,shi2023diffusion}, which extend \cite{bortoli2021diffusion} but differ with respect to the projection sets used to define their half-bridge formulations, and \cite{pooladian2025plugin}, which derives a way to estimate the drift of the Schr\"odinger bridge from the corresponding static coupling. Schr\"odinger-bridge-based methods alleviate the computational expense incurred by score-based generative models (SGM) \cite{bortoli2021diffusion}, which require the forward diffusion process to run for longer times with smaller step sizes. Unlike SGM, our method provides a tool to flow an existing sample in the same data distribution with control over the spread of the newly obtained sample.

To the best of our knowledge, the work  \cite{albergo2023stochastic} is the only one in the literature on generative modeling that 
considers 
maps from a distribution to itself. In their paper, flow matching learns a drift function associated with a stochastic path from the data distribution to itself. Their stochastic interpolants, however, do not attempt to seek optimality in the relative entropy sense, a property correlated to sampling effectiveness and generation quality \cite{shaul2023kinectic}. By contrast, our method seeks to estimate the coupling with minimal relative entropy, similar in spirit to methods such as \cite{bortoli2021diffusion,shi2023diffusion,vargas2021likelihood}; our method, however, presents certain algorithmic advantages over these, which can only be derived for the mirror case.

\section{Mathematical Preliminaries}
\label{sec:preliminaries}

In this section we recall the formulation of the Schr\"odinger bridge problem on path space, specializing to the case of diffusion processes. We then review the iterative proportional fitting procedure (IPFP) and its implementation in terms of drift functions. This will provide the basic probabilistic and analytic tools needed for our time-symmetric formulation in Sections~\ref{sec:method}--\ref{sec:symmipf}.

Throughout, we write $\R^n$ for $n$-dimensional Euclidean space, and we use $\|\cdot\|_2$ for the usual Euclidean norm. For any probability measure $\sP$ on a measurable space, we write $\E_{\sP}[\cdot]$ to denote the expectation with respect to $\sP$.

\subsection{Definitions} Let $n > 0$ be an integer, and let $C([0,1],\R^n)$ denote path space, i.e., the space of continuous functions from $[0,1]$ to $\R^n$. We write $\mathcal{P}(C([0,1],\R^n))$ for the space of Borel probability measures on the path space $C([0,1],\R^n)$; its elements are known as \emph{path measures}. For each $t\in[0,1]$, we define the \emph{time-evaluation map}
$${\rm eval}_t : C([0,1],\R^n)\to\R^n,\qquad {\rm eval}_t(\omega)\coloneqq \omega(t).$$
Given a path measure $\sP \in \mathcal{P}(C([0,1],\R^n))$, we denote by
$$\sP_t \coloneqq ({\rm eval}_t)_{\#}\sP$$
its time-$t$ marginal, i.e., the pushforward of $\sP$ under ${\rm eval}_t$. Thus $\sP_t$ is a probability measure on $\R^n$. When $\sP_t$ is absolutely continuous with respect to Lebesgue measure, we can make sense of its corresponding probability density function, which we denote by $\rho_t^{\sP}$.

Next, let $\sP^0 \in \mathcal{P}(C([0,1],\R^n))$ be a fixed \emph{reference path measure}. Following \cite{benton1975Markov,leonard2014survey}, we formulate the \emph{Schr\"odinger bridge problem} as 
the problem of finding a path measure $\sP_{\rm{SB}} \in \mathcal{P}(C([0,1],\R^n))$ interpolating between prescribed initial and final marginals $\pi_0$ and $\pi_1$ that is the closest to  
$\sP^0$ with respect to 
Kullback--Leibler divergence $\KL$. To be precise, given probability measures $\pi_0$ and $\pi_1$ on $\R^n$, we define the constraint set
$$\sD(\pi_0,\pi_1) \coloneqq \big\{ \sP \in \mathcal{P}(C([0,1],\R^n)) : \sP_0 = \pi_0,\ \sP_1 = \pi_1 \big\}.$$
For $\sP,\sQ \in \mathcal{P}(C([0,1],\R^n))$ with $\sP$ absolutely continuous with respect to $\sQ$, we write
$$\KL(\sP \;\Vert\; \sQ) \coloneqq \int_{C([0,1],\R^n)} \log\Big(\frac{\mathrm d\sP}{\mathrm d\sQ}\Big)\,\mathrm d\sP$$
for the Kullback--Leibler (KL) divergence between $\sP$ and $\sQ$, and we set $\KL(\sP\;\Vert\;\sQ) = +\infty$ otherwise. With this notation, the Schr\"odinger bridge $\sP_{\rm{SB}}$ is defined as the solution of the optimization problem
\begin{equation} \label{eq:sbp}
    \sP_{\rm{SB}}\coloneqq\argmin_{\sP\in\sD(\pi_0,\pi_1)} \KL\left( \sP\;\Vert\;\sP^0\right).
\end{equation}
In other words, $\sP_{\rm{SB}}$ is the \emph{direct KL projection} of $\sP^0$ onto the constraint set $\sD(\pi_0,\pi_1)$. In the terminology of \cite{leonard2014survey}, the formulation (\ref{eq:sbp}) is the \emph{dynamic} Schr\"odinger bridge problem, in contrast with the corresponding \emph{static} formulation posed directly on couplings between the endpoint marginals $\pi_0$ and $\pi_1$.

In this paper we consider path measures that arise as the laws of diffusion processes. A diffusion process is a stochastic process $(\rmX_t)_{t\in[0,1]}$ taking values in $\R^n$ and governed by a forward stochastic differential equation (SDE) of the form
$$\rmd\rmX_t=f_t(\rmX_t) \rmd t+\sigma \rmd \rmW_t,\qquad t\in[0,1],$$
where $f_t:\R^n\to\R^n$ is the \emph{forward drift} at time $t$, $\sigma > 0$ is the noise coefficient, and $(\rmW_t)_{t \in [0,1]}$ is a standard $n$-dimensional Wiener process. Under standard regularity assumptions on $f_t$, e.g., global Lipschitz continuity and linear growth, specifying the initial law $\rmX_0 \sim \pi_0$ uniquely determines the law of the process on $C([0,1],\R^n)$ and hence a path measure $\sP \in \mathcal{P}(C([0,1],\R^n))$. In this case, we use a superscript $\sP$ as in $f_t^{\sP}$ to denote the forward drift of the corresponding stochastic process.

An important feature of diffusion processes is that their time-reversals are again diffusion processes with the same noise coefficient $\sigma$. More precisely, let $(\rmX_t)_{t\in[0,1]}$ solve the forward SDE above, and define the time-reversed process $(\rmY_t)_{t\in[0,1]}$ by $\rmY_t \coloneqq \rmX_{1-t}$ for $t\in[0,1]$. 
Then the family of random variables $(\rmY_t)_{t\in[0,1]}$ has the same joint law as $(\rmX_{1-t})_{t\in[0,1]}$, and $\rmY$ solves a backward SDE of the form
$$\rmd\rmY_t = b_t(\rmY_t) \rmd t + \sigma \rmd \rmW_t,$$
here $b_t:\R^n\to\R^n$ is the \emph{backward drift} at time $t$; see, e.g., \cite[\S2.3]{winkler2023score} for a detailed discussion. We write $b_t^{\sP}$ to denote the backward drift of the stochastic process with forward drift $f_t^{\sP}$. The forward and backward drifts $f_t^{\sP}$ and $b_t^{\sP}$ are related by the identity known as Nelson's relation:
\begin{equation} \label{eq:nelson}
f_t^{\sP}(x) - b_{1-t}^{\sP}(x) = \sigma^2 \nabla \log \rho_t^{\sP}(x)
\end{equation}
for all $t\in[0,1]$ and all $x \in \R^n$ such that $\rho_t^{\sP}(x)>0$.

Given a drift field $v_t : \R^n \to \R^n$ and an initial density $\rho_0^v = \pi_0$, we write $\rho_t^v$ for the time-$t$ marginal density of the diffusion process with drift $v_t$ and noise coefficient $\sigma$. The marginals $\rho_t^v$ evolve according to a linear partial differential equation called the \emph{Fokker--Planck} equation. This equation expresses conservation of probability mass under the combined effect of drift and diffusion and is stated as follows:
\begin{equation} \label{eq:fp}
\partial_t \rho_t^v = - \nabla \cdot (\rho_t^v v_t) + \frac{\sigma^2}{2} \Delta \rho_t^v,\qquad \rho_0^v = \pi_0.
\end{equation}
When $\sP$ is the law of the diffusion with forward drift $v_t$ and initial density $\pi_0$, we write $\rho_t^{\sP} = \rho_t^v$.

Finally, throughout this paper we assume that the reference path measure $\sP^0$ in (\ref{eq:sbp}) arises as the law of a diffusion process with noise coefficient $\sigma$. Under this assumption, any path measure $\sP$ with finite KL divergence with respect to $\sP^0$, including the Schr\"odinger bridge $\sP_{\rm{SB}}$, necessarily arises from a diffusion process with the same noise coefficient $\sigma$ \cite{leonard2014measures,vargas2021likelihood}. Consequently, the optimization in (\ref{eq:sbp}) may be restricted to diffusion path measures without loss of generality. Moreover, by adjusting the initial condition of the reference SDE, we may assume that $\sP^0_0 = \pi_0$, i.e., that the reference process has the desired initial marginal, without changing the solution of (\ref{eq:sbp}).

\subsection{Iterative Proportional Fitting Procedure} \label{sec:ipfp}

The typical strategy for solving the problem (\ref{eq:sbp}) is to apply a general technique known as the \emph{Iterative Proportional Fitting Procedure} (IPFP) \cite{fortet1940res,kullback1968probabilities}. Given initial and final marginals $\pi_0$ and $\pi_1$, IPFP constructs a sequence of path measures $(\sP^k)_{k\ge 0}$ by iteratively solving the pair of \emph{half-bridge} problems
\begin{equation} \label{eq:half-bridges}
\sP^{2k+1}=\argmin_{\sP\in\sD(\cdot,\pi_1)}\KL\left( \sP\;\Vert\;\sP^{2k} \right),\quad\sP^{2k+2}=\argmin_{\sP\in\sD(\pi_0,\cdot)}\KL\big( \sP\;\Vert\;{\sP^{2k+1}}\big),    
\end{equation}
where $\sD(\cdot,\pi_1)$ (resp., $\sD(\pi_0,\cdot)$) 
denotes the space of path measures with final (resp., initial) marginal 
$\pi_1$ (resp., $\pi_0$). 
Intuitively, each update in (\ref{eq:half-bridges}) performs a direct KL projection onto one of the constraint sets while staying as close as possible, in the sense of KL divergence, to the previous iterate. IPFP can be thought of as an extension of Sinkhorn's algorithm to continuous state spaces, where the rescaling updates characteristic of Sinkhorn are replaced by iterated direct KL projections onto sets of distributions with fixed initial or final marginal \cite{essid2019trav}.

Direct KL projections satisfy a Pythagorean theorem, from which it follows that the IPFP iterates converge to the Schr\"odinger bridge. More precisely, we have the following result:

\begin{proposition}[\protect{\cite{ruschendorf1995iterative}}]
    The sequence $(\sP^k)_{k\ge 0}$ defined by $($\ref{eq:half-bridges}$)$ converges in total variation to the Schr\"odinger bridge $\sP_{\rm{SB}}$ solving $($\ref{eq:sbp}$)$.
\end{proposition}

In the setting of diffusion processes, it is convenient to express the IPFP updates in terms of drift functions rather than path measures, as drift functions are more amenable to modeling and estimation than path measures. To this end, let $\sP \in \sD(\pi_0,\cdot)$ and $\sP^\dagger \in \sD(\cdot,\pi_1)$ be path measures corresponding to diffusion processes with the same noise coefficient $\sigma$. We write $f_t^{\sP}$ and $b_t^{\sP}$ for the forward and backward drifts associated with $\sP$, and similarly $f_t^{\sP^\dagger}$ and $b_t^{\sP^\dagger}$ for $\sP^\dagger$. As a consequence of Girsanov's theorem, the KL divergence between $\sP$ and $\sP^\dagger$ admits the following representation in terms of these drift functions, see e.g., \cite[\S3]{Chen2016control} and \cite[\S2.2, \S2.3]{winkler2023score}:
\begin{align}
\KL(\sP \;\Vert\; \sP^\dagger) &= \KL\big(\sP_{0} \;\Vert\; \sP^\dagger_{0}\big) + \frac{1}{2\sigma^2}\int_0^1 \E_{\sP}\left[\big\|f_t^{\sP}(\rmX_t) - f_t^{\sP^\dagger}(\rmX_t)\big\|_2^2\right]  \rmd t \label{eq:klviadrift1}\\
&= \KL\big(\sP_{1} \;\Vert\; \sP^\dagger_{1}\big) +  \frac{1}{2\sigma^2}\int_0^1 \E_{\sP}\left[\big\|b_t^{\sP}(\rmX_t) - b_t^{\sP^\dagger}(\rmX_t)\big\|_2^2\right]  \rmd t, \label{eq:klviadrift2}
\end{align}
where $(\rmX_t)_{t\in[0,1]}$ denotes the diffusion process with law $\sP$. The identities (\ref{eq:klviadrift1})--(\ref{eq:klviadrift2}) show that, up to boundary terms at $t=0$ or $t=1$, the KL divergence between diffusion path measures can be expressed as a time-integral of squared drift differences.

We begin with the direct projection onto the set of path measures with a prescribed initial marginal, which is characterized in terms of drift functions as follows:

\begin{proposition} \label{prop:directmarginal}
Let $\sP^\dagger \in \sD(\cdot,\pi_1)$ be a diffusion path measure with forward drift $f_t^{\sP^\dagger}$. The direct $\KL$ projection of $\sP^\dagger$ onto the space $\sD(\pi_0,\cdot)$ is the unique path measure $\sP \in \sD(\pi_0,\cdot)$ whose forward drift satisfies $f_t^{\sP} = f_t^{\sP^\dagger}$ for all $t\in[0,1]$.
\end{proposition}

In other words, to project onto $\sD(\pi_0,\cdot)$ we simply replace the initial marginal of $\sP^\dagger$ by $\pi_0$ while keeping its forward drift unchanged. In IPFP (\ref{eq:half-bridges}), one applies Proposition~\ref{prop:directmarginal} at each odd step by taking $\sP = \sP^{2k+1} \in \sD(\pi_0,\cdot)$ to have forward drift equal to that of $\sP^\dagger = \sP^{2k}$.

Given a probability distribution $\pi$ on $\R^n$, let $v_t \mapsto \gJ_\pi[v_t]$ denote the operator that maps a forward drift $v_t$ to the corresponding backward drift of the induced diffusion process with initial marginal given by $\pi$. For instance, $\gJ_\pi[v_t]$ can be characterized in terms of Nelson's relation (\ref{eq:nelson}), where the marginal densities $\rho_t^v$ are governed by the Fokker--Planck equation (\ref{eq:fp}). Writing $f_t^{2k}$ and $b_t^{2k}$ for the forward and backward drifts associated with $\sP^{2k}$, the IPFP iterates (\ref{eq:half-bridges}) may be expressed purely in terms of drifts as
\begin{equation} \label{eq:half-bridges-drifts}
b_t^{{2k+1}} =  \gJ_{\pi_1}\big[f_t^{{2k}}\big],\qquad 
f_t^{{2k+2}} = \gJ_{\pi_0}\big[b_t^{{2k+1}}\big].
\end{equation}
Consequently, IPFP alternates between computing a backward drift from a forward drift with prescribed initial marginal, and computing a forward drift from a backward drift with prescribed final marginal.

\subsection{Applications} 

The Schr\"odinger bridge framework provides a flexible tool for generative modeling. Suppose $\pi_0$ is given by a data distribution $\pdata$, and let $\pi_1$ be an easy-to-sample distribution $\pprior$, for example a standard Gaussian distribution $\gN(0,\mI)$. The backward diffusion process associated with $\sP_{\rm{SB}}$ then yields a generative model for sampling from $\pdata$: one first samples from $\pprior$ at time $t=1$ and then integrates the backward SDE to obtain approximate samples from $\pdata$ at time $t=0$.

In practice, the IPFP iterates (\ref{eq:half-bridges}) can be approximated by an algorithm known as the \emph{Diffusion Schr\"odinger Bridge} (DSB), developed in \cite{bortoli2021diffusion}, which leverages Proposition \ref{prop:directmarginal} to train neural networks $f_t^\theta$ and $b_t^\phi$ to estimate the forward and backward drifts of the Schr\"odinger bridge $\sP_{\rm{SB}}$.  In the remainder of the paper, we will specialize to a time-symmetric setting $\pi_0=\pi_1$ with a time-symmetric reference path measure $\sP^0$. DSB does not exploit this special structure: the algorithm still learns two separate drift networks $f_t^\theta$ and $b_t^\phi$ and alternates between forward and backward updates as in the general case. Moreover, in practice, we find that the estimate of the Schr\"odinger bridge obtained by DSB fails to respect the time symmetry induced by the symmetric choice of marginal data and reference process.

In \S\ref{sec:method}, we will formulate a symmetrized version of IPFP; this reformulation leads to a symmetrized version of DSB that encourages time symmetry, uses only one neural network, and enjoys substantial computational savings, as we explain in \S\ref{sec:symmipf}.

\section{Mirror Schr\"odinger Bridges}
\label{sec:method}

We now specialize the Schr\"odinger bridge problem to what we call the \emph{mirror} setting, where both the initial and final marginals $\pi_0$ and $\pi_1$ are equal to a given target distribution $\pi$ and where the reference path measure is symmetric in time. Our goal in this section is to formalize the notion of time-symmetry for path measures, define the mirror Schr\"odinger bridge, and record  basic structural properties that will be crucial in \S\ref{sec:symmipf}.

Retain the setting and notation of \S\ref{sec:preliminaries}. We take $\pi_0 = \pi_1 = \pi$, and we take the reference path measure $\sP^0$ to be time-symmetric in the sense made precise as follows:

\begin{definition}
    Let $R$ be the time-reversal involution on $C([0,1],\R^n)$ defined by $(R\omega)(t) \coloneqq \omega(1-t)$. Then a path measure $\sP \in \mathcal{P}(C([0,1],\R^n))$  is time-symmetric if and only if $\sP\circ R = \sP$, i.e., if $\sP_t = \sP_{1-t}$. We denote the set of time-symmetric path measures by
$$\sS \coloneqq \big\{\sP \in \mathcal{P}(C([0,1],\R^n)) : \sP\circ R = \sP\big\}.$$
\end{definition}

When $\sP$ arises as the law of a diffusion process $(\rmX_t)_{t \in [0,1]}$, the notion of time-symmetry has a simple dynamical interpretation in terms of its forward and backward drifts $f_t^\sP$ and $b_t^\sP$. Indeed, we have the following basic result:

\begin{lemma} \label{lem:timesymmetry}
    If $\sP \in \sS$ is the law of a diffusion process, then for almost every $t \in [0,1]$, we have $f_t^{\sP}(x) = b_t^{\sP}(x)$ for $x$ in a set of full measure under $\sP_t$. In particular, if $\rho_t^{\sP}$ is smooth and strictly positive and $f_t^{\sP}$ and $b_t^{\sP}$ are continuous, then $f_t^{\sP}$ and $b_t^{\sP}$ agree everywhere.
\end{lemma}
\begin{proof}
Let $\rmX$ be a diffusion process with law $\sP$ and let $\widetilde{\rmX}$ be its time-reversal defined by $\widetilde{\rmX}_t  = \rmX_{1-t}$ for all $t \in [0,1]$. Time-symmetry of $\sP$ implies that the laws of $\rmX$ and $\widetilde{\rmX}$ coincide. Moreover, the backward drift $b_t^{\sP}$ of $\rmX$ is the forward drift of $\widetilde{\rmX}$. On the other hand, the infinitesimal generator $L_{t,f}$ of $\rmX$ acts on smooth and compactly supported test functions $\phi$ via
 \begin{equation} \label{eq:gen1}
 (L_{t,f} \phi)(x) = f_t^{\sP}(x) \cdot \nabla \phi(x) + \frac{\sigma^2}{2}\Delta \phi(x),
 \end{equation}
 whereas the infinitesimal generator $L_{t,b}$ of $\widetilde{\rmX}$ acts via
  \begin{equation} \label{eq:gen2}
  (L_{t,b} \phi)(x) = b_t^{\sP}(x) \cdot \nabla \phi(x) + \frac{\sigma^2}{2}\Delta \phi(x).
  \end{equation}
 Since $\rmX$ and $\widetilde{\rmX}$ have the same law, their generators must agree on smooth and compactly supported functions, so comparing (\ref{eq:gen1}) and (\ref{eq:gen2}) yields that $f_t^{\sP} = b_t^{\sP}$ on a set of full $\sP_t$-measure for almost every $t$. Under the additional regularity assumptions stated in the lemma, the equality extends to all $x\in\R^n$. Alternatively, under such regularity assumptions the equality $f_t^{\sP} = b_t^{\sP}$ follows directly from Nelson's relation (\ref{eq:nelson}).
\end{proof}
We conclude from Lemma~\ref{lem:timesymmetry} that for a time-symmetric path measure arising from a diffusion process, the forward and backward drifts coincide almost everywhere under the path measure. This observation will play a central role in our approach to solving the Schr\"odinger bridge problem in the mirror setting. This problem is stated as follows:
\begin{definition}
The mirror Schr\"odinger bridge is the solution $\sP_{\rm{MSB}}$ to the Schr\"{o}dinger bridge problem with identical initial and final marginals $\pi_0 = \pi_1 = \pi$ with respect to a time-symmetric reference measure $\sP^0$:
\begin{equation} \label{eq:msb}
    \sP_{\rm{MSB}}\coloneqq\argmin_{\sP\in\sD(\pi,\pi)} \KL\left( \sP\;\Vert\;\sP^0\right).
\end{equation}
\end{definition}
In our applications we typically take $\sP^0$ to be the law of an Ornstein--Uhlenbeck process $(\rmX_t)_{t\in[0,1]}$ in stationarity, i.e., the solution of the SDE
$$\rmd\rmX_t=-\alpha\rmX_t \rmd t+\sigma\rmd \rmW_t,\qquad \alpha>0,$$
with initial law given by its stationary Gaussian distribution. This is a particularly common choice of reference process in the literature, and it has the special property that it is reversible, in the sense that its law is invariant under time-reversal and its stationary marginal is invariant under the dynamics. In particular, it is time-symmetric, so $\sP^0 \in \sS$. More generally, we could take $\sP^0$ to be the law of any reversible diffusion process with stationary marginal $\pi$, although such processes form a strict subclass of $\sS$. In contrast, standard Brownian motion on $\R^n$ is not reversible with respect to a fixed marginal and does not yield a time-symmetric path measure in the sense above, so it does not fit into our mirror setting.

As one might expect, the mirror Schr\"odinger bridge inherits the time-symmetry of the reference process $\sP^0$:

\begin{lemma} \label{lem:mirrorsymmetry}
    We have that $\sP_{\rm{MSB}} \in \sS$.
\end{lemma}
\begin{proof}
     Since $\sP^0$ is time-symmetric, we have $\sP^0 \circ R = \sP^0$. By the definition of KL divergence on path space, the Schr\"odinger bridge associated with the pair $(\pi,\pi)$ and reference $\sP^0\circ R$ is $\sP_{\rm{MSB}}\circ R$. But $\sP^0\circ R = \sP^0$, so by uniqueness of the Schr\"odinger bridge we must have $\sP_{\rm{MSB}} = \sP_{\rm{MSB}} \circ R$, which is precisely the statement that $\sP_{\rm{MSB}} \in \sS$.
\end{proof}
Combining Lemmas \ref{lem:timesymmetry} and \ref{lem:mirrorsymmetry}, we see that the forward and backward drifts of the diffusion process associated with $\sP_{\rm{MSB}}$ are almost surely equal under $\sP_{\rm{MSB}}$, with exact equality under mild regularity assumptions. For ease of notation, we denote this common drift function by $v_t^*$, which we understand to represent both the forward and the backward drift of $\sP_{\rm{MSB}}$.

A na\"ive approach to solving the mirror Schr\"odinger bridge problem (\ref{eq:msb}) is to apply IPFP with both marginals $\pi_0=\pi_1$  equal to $\pi$.  In practice, this requires iterative training of two neural networks $f_t^\theta$ and $b_t^\phi$, the first modeling the drift of the forward process associated to $\sP_{\rm{MSB}}$ and the latter modeling the corresponding backward drift. But this straightforward application of IPFP leads to unnecessary computational expense, as it fails to use the time-symmetry of the mirror problem (\ref{eq:msb}). To our knowledge, no prior approach has systematically leveraged symmetry in the dynamical formulation at the level of path measures and drifts. Our goal in the next section is to reformulate the IPFP iterations so as to operate directly on a single drift function and to modify the updates to explicitly encourage time-symmetry.

\section{Symmetrized Iterative Proportional Fitting}
\label{sec:symmipf}

Recall from \S\ref{sec:ipfp} that classical IPFP alternates between two distinct KL projections, one onto the set of path measures with prescribed initial marginal $\pi_0$, and one onto the set of path measures with prescribed final marginal $\pi_1$. In the language of drifts, these two updates can be expressed in terms of the operators $\gJ_{\pi_0}$ and $\gJ_{\pi_1}$, as in (\ref{eq:half-bridges-drifts}). In the mirror setting $\pi_0 = \pi_1 = \pi$, and the two projection steps are implemented by the same operator $\gJ_\pi$. Consequently, at the level of drifts, IPFP reduces in the mirror setting to performing the following fixed point iteration:
\begin{align}
    &v^{k+1}=\gJ_\pi\left[v^{k}\right] \label{eq:direct-kl-1}
\end{align}
where $v_t^0$ is the (forward and backward) drift of the time-symmetric reference process $\sP^0 \in \sS$. This suggests the strategy of modeling the drift with a single neural network and using the fixed point iteration (\ref{eq:direct-kl-1}) to recursively train that network. Even without further modification, this strategy already yields a significant simplification in the implementation of IPFP: because the inner iterations are the same, it suggests that the mirror Schr\"odinger bridge can be estimated in roughly half the number of training iterations as standard IPFP. 

Nevertheless, the map $v \mapsto \gJ_\pi[v]$ is not explicitly time-symmetric, even though its fixed point $v^*$ is. This motivates us to formulate a symmetrized version of (\ref{eq:direct-kl-1}) that directly promotes time-symmetry at each iteration. This encourages the learned dynamics to respect the time-symmetric structure of the true mirror Schr\"odinger bridge by damping out time-asymmetric components of the estimation error.

\subsection{Linearization at the Mirror Schr\"{o}dinger Bridge} \label{sec:linear}

Our key observation is that, at least to first order around the optimal drift $v^*$, symmetrization can be achieved by averaging the forward and backward drifts associated with a given diffusion process. The resulting method can be interpreted as a convenient first-order approximation of a fully symmetrized version of IPFP, in the spirit of the symmetrization procedure developed by Kurras \cite{kurras2015symmetric} to handle the mirror case of the static Schr\"odinger bridge problem.

To state our first-order approximation result, we require two preliminary ingredients. The first is to specify the functional analytic setting in which we will work. Let $(\rmX_t)_{t\in[0,1]}$ denote the diffusion process with law $\sP_{\rm MSB}$ and drift $v_t^*$, and let $\rho_t^* \coloneqq \rho_t^{v^*}$ denote its time-marginal densities. We consider drift perturbations $h_t:\R^n\to\R^n$ that are square-integrable along the trajectories of $\sP_{\rm MSB}$, and we equip this space with the norm
\begin{equation} \label{eq:drift-L2-norm}
\|h\|_2^2 \coloneqq \int_0^1 \E_{\sP_{\rm MSB}}\bigg[\|h_t(\rmX_t)\|_2^2\bigg]\,\rmd t
= \int_0^1 \int_{\R^n} \|h_t(x)\|_2^2 \,\rho_t^*(x)\,\rmd x\,\rmd t.
\end{equation}
The completion of smooth, compactly supported perturbations with respect to this norm is a Hilbert space, which we denote by $\gB$. Elements of $\gB$ will serve as admissible directions for linearizing quantities at $v^*$.

The second ingredient is to define the \emph{drift-averaging operator}
$$\gA(v) \coloneqq \frac{1}{2}\big(v + \gJ_\pi[v]\big).$$
By construction, the optimal drift $v^*$ is a fixed point of $\gA$, as it is a fixed point of $\gJ_\pi$. Having defined the space $\gB$ and the operator $\gA$, we are now in position to state the main result of this section, which characterizes the Fr\'{e}chet derivatives $J\coloneqq D\gJ_\pi[v^*]$ and $A \coloneqq D\gA[v^*]$ of the operators $\gJ_\pi$ and $\gA$ at the optimal drift $v^*$:

\begin{proposition} \label{prop:reorder}
The map $J \coloneqq D\gJ_\pi[v^*]$ is an involution, i.e., $J^2 = \mathrm{Id}$. Thus, the space $\gB$ decomposes as a topological direct sum $\gB = E_+ \oplus E_-$ of the $+1$ and $-1$ eigenspaces of $J$, denoted by $E_{\pm} \subset \gB$, and the map $A \coloneqq D\gA[v^*]$ is the projection onto $E_+$ along $E_-$.
\end{proposition}

As a consequence of the decomposition $\gB = E_+ \oplus E_-$, we can think of each element of $\gB$ as having a time-symmetric component, namely its $E_+$-component, and a time-antisymmetric component, namely its $E_-$-component. In this language, the derivative $A = D\gA[v^*]$ is the projection onto the time-symmetric component. Therefore, in a neighborhood of $v^*$, the nonlinear map $\gA$ acts to first order as a symmetry-enforcing transformation on drift functions.

\begin{proof}[Proof of Proposition \ref{prop:reorder}]
We proceed by perturbing the optimal drift $v^*$ and computing the resulting first-order change in the time-reversal update $\gJ_\pi$. For this purpose, we find it useful to split drift functions into their \emph{osmotic} and \emph{current} velocity components.
\begin{definition}
For a given drift $v \in \gB$ and the associated densities $(\rho^v_t)_{t\in[0,1]}$, the corresponding \emph{osmotic} and \emph{current} velocities are defined by
$$u_t^v \coloneqq \frac{\sigma^2}{2} \nabla \log \rho_t^v, \qquad w_t^v \coloneqq v_t - u_t^v.$$
\end{definition}

The osmotic velocity encodes the contribution of the density evolution to the drift, while the current velocity captures the remaining transport component. Now fix $h \in \gB$, viewed as a perturbation direction, and for small $\varepsilon$ consider the perturbed drift
$$v_t^\varepsilon \coloneqq v_t^* + \varepsilon h_t$$
with associated marginal probability density $\rho_t^\varepsilon \coloneqq \rho_t^{v^*+\varepsilon h}$. Then $\rho_t^\varepsilon$ solves the perturbed Fokker--Planck equation
\begin{equation} \label{eq:perturbed-FP}
\partial_t \rho_t^\varepsilon = - \nabla \cdot \big(\rho_t^\varepsilon (v_t^* + \varepsilon h_t)\big) + \frac{\sigma^2}{2} \Delta \rho_t^\varepsilon,\qquad \rho_0^\varepsilon = \pi.
\end{equation}
We define the first-order variation of the marginal density by
$$\eta_t \coloneqq \left.\frac{\mathrm d}{\mathrm d\varepsilon} \rho_t^\varepsilon \right|_{\varepsilon = 0}.$$
Differentiating (\ref{eq:perturbed-FP}) at $\varepsilon=0$ and expressing the result in terms of $\eta_t$ yields the following linear parabolic equation:
\begin{equation} \label{eq:linpar}
\partial_t \eta_t = -\nabla \cdot (\eta_t v_t^* + \rho_t^* h_t) + \frac{\sigma^2}{2}\Delta\eta_t,\qquad \eta_0 = 0,
\end{equation}
where $\rho_t^* \coloneqq \rho_t^{v^*}$ denotes the marginal density of the mirror Schr\"odinger bridge $\sP_{\rm{MSB}}$.

It is natural to measure the size of density perturbations in the $L^2$-norm
\begin{equation} \label{eq:density-L2-norm}
\|\eta\|_{2}^2 \coloneqq \int_0^1 \int_{\R^n} |\eta_t(x)|^2 \,\rho_t^*(x)\,\rmd x\,\rmd t,
\end{equation}
which endows the space of such perturbations the structure of Hilbert space. The next lemma records the basic regularity of the map $h \mapsto \eta$ of Hilbert spaces.
\begin{lemma} \label{lem:h-to-eta-bounded}
    The mapping $h_t \mapsto \eta_t$ defined by (\ref{eq:linpar}) is linear, and under mild regularity assumptions, it is bounded. More precisely, assume that $\rho_t^*(x)>0$ for all $t\in[0,1]$ and $x\in\R^n$, that there exist finite constants $M,B>0$ such that $\rho_t^*(x) \leq M$ and $\|\nabla\log\rho_t^*\|_{L^\infty} \leq B$, and that the drift $v_t^*$ has absolutely bounded norm. Further assume $\eta_t$ and $\nabla\eta_t$ are square-integrable with respect to $\rho_t^*(x)\,\rmd x$ for all $t$. Then there exists a constant $C>0$, depending only on $M$, $B$, $\sigma$, and $\|v^*\|_{L^\infty}$, such that $\|\eta\|_2 \leq C \| h\|_2$.
\end{lemma}
\begin{proof}[Proof of Lemma \ref{lem:h-to-eta-bounded}]
Observe that (\ref{eq:linpar}) is linear in both $\eta_t$ and $h_t$, which implies that the mapping $h_t \mapsto \eta_t$ is linear. As for boundedness, we reserve the proof to the supplementary material.
\end{proof}

With this control on the density perturbations in hand, we now calculate the corresponding first-order variations of the osmotic and current velocities. Let $(\delta u_t,\delta w_t)$ denote the perturbations of $(u_t^*, w_t^*) \coloneqq (u_t^{v^*},w_t^{v^*})$ induced by $h_t$ via
\begin{equation}\label{eq:du-dw-def}
\delta u_t \coloneqq \left.\frac{\mathrm d}{\mathrm d\varepsilon} u_t^{v^*+\varepsilon h}\right|_{\varepsilon=0}
= \frac{\sigma^2}{2} \nabla \left(\frac{\eta_t}{\rho_t^*}\right), 
\qquad 
\delta w_t \coloneqq h_t - \delta u_t.
\end{equation}
The expression for $\delta u_t$ is obtained by differentiating the osmotic velocity formula $u_t^v = \frac{\sigma^2}{2}\nabla\log\rho_t^v$, while the definition of $\delta w_t$ follows from the decomposition $v_t = u_t^v + w_t^v$ and the fact that the total drift variation is $h_t$.

Going forward, it will be convenient to modify our $L^2$-norm on $\gB$ slightly so that it is expressed in terms of the perturbations $(\delta u, \delta w)$. Specifically, we define
\begin{equation} \label{eq:newnorm}
 \| (\delta u, \delta w) \|_2^2 \coloneqq \int_0^1 \E_{\sP_{\mathrm{MSB}}}\left[\|\delta u_t(\rmX_t)\|_2^2 + \|\delta w_t(\rmX_t)\|_2^2\right]\rmd t.
\end{equation}
This new norm majorizes the norm $\frac{1}{\sqrt{2}}\|h\|_2$ on a neighborhood of $v^*$; in particular, it follows from Lemma \ref{lem:h-to-eta-bounded} that the composite map $(\delta u, \delta w) \mapsto h \mapsto \eta$ is bounded.

By construction, the norm (\ref{eq:newnorm}) is invariant under time-reversal: since $\sP_{\rm MSB} \in \sS$ and the quadratic form inside the expectation is symmetric under $(\delta u_t,\delta w_t)\mapsto(\delta u_{1-t},-\delta w_{1-t})$, the value of $\|h\|_2$ is unchanged when $h_t$ is replaced by its time-reversed counterpart. 

We can now compute the action of $J$ on $h$ in terms of $(\delta u, \delta w)$. The time-reversal transformation sends the pair $(u_t^{*},w_t^{*})$ to $(u_{1-t}^{v^*},-w_{1-t}^{v^*})$, i.e., it leaves the osmotic component invariant as $u_{1-t}^{v^*}=u_t^{v^*}$ by Lemma \ref{lem:mirrorsymmetry}, whilst reversing the time and sign of the current component. Thus, the action of $J$ on $h$ can be written as
\begin{equation}\label{eq:J-action}
(Jh)_t = \delta u_{1-t} - \delta w_{1-t}.
\end{equation}
Applying (\ref{eq:J-action}) twice yields that
$$
(J^2 h)_t(x) = J\big( (\delta u - \delta w)_{\cdot}(\cdot) \big)_t(x)
= \big(\delta u - (-\delta w)\big)_{1-(1-t)}(x) = (\delta u + \delta w)_t(x) = h_t(x),
$$
implying that $J^2 = \mathrm{Id}$. It follows that $J$ has eigenvalues $1$ and $-1$; denote the corresponding eigenspaces by $E_{\pm}$. Since $A = \tfrac12(\mathrm{Id} + J)$, we have that $A$ is the identity on $E_+$ and vanishes on $E_-$. Equivalently, $A$ is the projection onto $E_+$ along $E_-$. 

Finally, it is clear that $\gB$ is the algebraic direct sum of $\ker A$ and $\operatorname{im} A$, as the inclusion $\operatorname{im} A \subset \gB$ defines a splitting. 
    It remains to show that $E_+$ and $E_-$ are closed in $\gB$, but this follows from the fact that $E_+ = \ker(J - \mathrm{Id})$ and $E_- = \ker(J + \mathrm{Id})$ are both kernels of bounded linear operators on $\gB$ with respect to the norm (\ref{eq:newnorm}). Indeed, boundedness follows from the invariance of $\sP_{\rm MSB}$ under time-reversal together with the fact that (\ref{eq:newnorm}) is defined by a quadratic form that is invariant under $(\delta u,\delta w)\mapsto(\delta u,-\delta w)$ and $t\mapsto1-t$.
\end{proof}

\subsection{Symmetrized Formulation} 

The results of \S\ref{sec:linear} motivate us to replace the fixed point iteration (\ref{eq:direct-kl-1}) by a time-symmetrized variant in which each update is composed with the averaging operator $\gA$. Specifically, we consider the following iteration:
\begin{align}
    &v_t^{k+1}= \mathcal{A}(v_t). \label{eq:direct-kl-2}
\end{align}
This scheme is an approximation in the dynamical context of the symmetrized version of Sinkhorn's algorithm described by \cite{kurras2015symmetric} for the static Schr\"odinger bridge problem.

The advantage of using the scheme (\ref{eq:direct-kl-2}) is that at each iteration, it suppresses the time-antisymmetric component of the drift to first order, pushing the iterates closer to the time-symmetric subspace $E_+ \subset \gB$. Thus we expect that the iteration (\ref{eq:direct-kl-2}) converges to a good approximation of the mirror Schr\"odinger bridge drift $v^*$. We refer to the limit of (\ref{eq:direct-kl-2}), when it exists, as the \emph{mirror bridge} associated with $\pi$ and reference $\sP^0$.

In general, the mirror bridge need not coincide exactly with the mirror Schr\"odinger bridge $\sP_{\rm MSB}$, but our theoretical and empirical results suggest that they are very close in the settings of interest. We confirm this heuristic in two different ways. First, in \S\ref{sec:symgauss}, we derive the exact dynamical analogue of Kurras' symmetrization step in the Gaussian case and prove that drift averaging is a first-order approximation. Second, we show that the limit of our approximate scheme is close to the Schr\"odinger bridge in the Gaussian case and is numerically indistinguishable from the result of IPFP, as implemented by DSB; see Fig.~\ref{fig:convergence} and the supplementary material.

In \S\ref{sec:algorithm}, we translate the symmetrized iteration (\ref{eq:direct-kl-2}) into a practical algorithm involving a single neural network modeling the drift of the diffusion process associated with $\sP_{\rm{MSB}}$. This algorithm enjoys essentially half the computational cost, in terms of training iterations, for the mirror problem when compared with IPFP-based algorithms like DSB that separately parameterize forward and backward drifts.

\section{Analysis of the Gaussian Case}

This section is devoted to analyzing our approach to mirror Schr\"odinger bridges in a Gaussian context. In \S\ref{sec:analytical}, we derive the mirror Schr\"odinger bridge for Gaussian marginals with Ornstein--Uhlenbeck reference. In \S\ref{sec:symgauss}, we keep this setting, derive the dynamical analogue of Kurras' symmetrization step, and compare it to our drift-averaging approximation, showing that the latter is in fact second-order accurate when the target marginal is a small perturbation of the stationary law of the reference. In \S\ref{sec:variation} we quantify how the noise level $\sigma$ controls in-distribution variation of target samples produced by the mirror Schr\"odinger bridge, validating our approach as a method for proximal sample generation.

\subsection{Analytical Solution in the Gaussian Case}
\label{sec:analytical}

\begin{figure}[t]
    \centering
    \includegraphics[width=\linewidth]{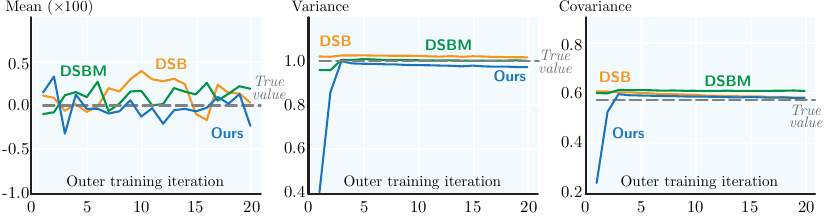}
    \vspace{-0.4cm}
    \caption{For each method, we plot the mean (left) and variance (middle) obtained for the terminal samples, i.e. samples obtained at time $t=T$, as well as the covariance (right) of the joint distribution, versus the number of outer iterations, averaged over 5 trials.}
    \label{fig:convergence}
\end{figure}

The following result computes the mirror Schr\"odinger bridge in the case of Gaussian marginals and Ornstein--Uhlenbeck reference.
\begin{proposition} \label{prop:analytical}
    Consider the static Schr\"{o}dinger bridge problem with initial and final marginals equal to the $d$-dimensional Gaussian distribution with zero mean and unit variance, where we take the reference measure $\pi^0$ corresponding to the Ornstein--Uhlenbeck process $\rmd\rmX_t = -\alpha \rmX_t \rmd t + \sigma \rmd\rmW_t$ running from $t = 0$ to $t = 1$. The solution $\pi^*$ to this problem is a $2d$-dimensional Gaussian with zero mean and covariance matrix $\Sigma$ given by
    $$\Sigma = \begin{pmatrix} \Sigma_{00} & \Sigma_{01} \\ \Sigma_{10} & \Sigma_{11} \end{pmatrix} = \begin{pmatrix} \mI & \beta\mI \\ \beta\mI & \mI \end{pmatrix},\quad \text{where} \quad \beta = \frac{\sigma^2(1-e^{2 \alpha }) + \sqrt{16 e^{2 \alpha } \alpha ^2+\sigma^4\left(1-e^{2 \alpha }\right)^2 }}{4 \alpha e^\alpha}.$$
\end{proposition}
We reserve the proof of Proposition \ref{prop:analytical} to the supplementary material. We use the proposition in \S\ref{sec:variation} to quantify the in-distribution variation of samples generated via the mirror Schr\"odinger bridge. We also use it as the ground truth for our Gaussian transport experiments; see \S\ref{sec:experiments}.

\subsection{Symmetrization in the Gaussian Case} \label{sec:symgauss}

Fix parameters $\alpha > 0$ and $\sigma > 0$, and let $\sP^0$ be the law on path space of the one-dimensional Ornstein--Uhlenbeck process
$\mathrm{d}\rmX_t = -\alpha \rmX_t \mathrm{d}t + \sigma\mathrm{d}W_t,$
started from its stationary marginal $\rho_0 = \mathcal{N}(0,q_0)$ with $q_0 := \sigma^2/(2\alpha)$. Then each time-marginal under $\sP^0$ is equal to $\mathcal{N}(0,q_0)$ and $\sP^0$ is time-symmetric. Let $\pi = \mathcal{N}(0,q_1)$ be a target marginal with variance $q_1 > 0$, which we regard as a small perturbation of $\rho_0$, so that
$q_1 = q_0(1+\varepsilon)$
for some small parameter $\varepsilon$.

Starting from $\sP^0$, a single outer IPFP iteration consists of taking the direct KL projection onto the set of paths with initial marginal $\pi$, and then the direct KL projection onto the set of paths with final marginal $\pi$. We denote the corresponding path measures by $\sP'$ and $\sP''$, and their forward drifts by $f_t^{\sP'}$ and $f_t^{\sP''}$. Kurras' symmetrization then forms a new, time-symmetric path measure $\sP^{\mathrm{sym}}$ built from $\sP'$ and $\sP''$; we write $f_t^{\mathrm{sym}}$ for its forward drift. The drift $f_t^{\mathrm{sym}}$ is in general difficult to compute, but it is possible to do so in the Gaussian context. By contrast, drift-averaging approximation instead uses the pointwise average
$$f_t^{\mathrm{avg}}(x)
:=
\frac{1}{2}\Bigl(f_t^{\sP'}(x) + f_t^{\sP''}(x)\Bigr).$$

\begin{proposition} \label{prop:gausserror}
    Retain the setting described above. Then, as $\varepsilon \to 0$, we have that
$$f_t^{\mathrm{sym}}(x) - f_t^{\mathrm{avg}}(x)
=
\mathcal{O}(\varepsilon^2)x$$
uniformly in $t \in [0,1]$. Thus, when $\pi$ is a small perturbation of the stationary marginal of $\sP^0$, drift-averaging yields a second-order accurate approximation of Kurras' symmetrization step.
\end{proposition}

We reserve the proof of Proposition \ref{prop:gausserror} to the supplementary material. This result provides a concrete instance of the statement that drift averaging is an accurate approximation to Kurras' symmetrization step.

\subsection{Sampling with In-Distribution Variation}
\label{sec:variation}

In this section, we provide a short intuitive explanation of how our method allows for resampling
with prescribed proximity to an input sample. Given such a sample $x_0 \sim \pi$, we solve the SDE corresponding to the Schr\"odinger bridge to push $x_0$ forward in time, arriving at a final sample $x_1 \in \pi$. We want $x_1$ to be a variation of $x_0$, where the proximity of $x_1$ to $x_0$ correlates with the size of the noise coefficient $\sigma$. Justifying this mathematically requires understanding how the conditional distribution $\rmX_1 \mid \rmX_0 = x_0$, specifically its mean and variance, depend on $\sigma$. While these quantities do not in general have closed form expressions, it is possible to compute them exactly in the case where $\pi = \gN(0,\rmI)$ is a $1$-dimensional Gaussian.

In this case, let $\rmX_t$ be the diffusion process associated to the Schr\"odinger bridge, where the reference path measure corresponds to an Ornstein--Uhlenbeck process with drift coefficient $-\alpha$. In Proposition \ref{prop:analytical}, we determine the joint distribution of $\rmX_0$ and $\rmX_1$ in terms of a quantity $\beta$, which is a function of $\alpha$ and $\sigma$ that grows approximately as $1 + c(\alpha) \cdot \sigma^2$ for some function $c$. Let $p(x,y)$ denote the probability density function of the joint distribution of $\rmX_0$ and $\rmX_1$, and recall that $p(x,y)$ is the product of the conditional PDF of $\rmX_1 \mid \rmX_0$ with the PDF of $\rmX_0$. Using this fact in conjunction with Proposition \ref{prop:analytical}, the PDF of $\rmX_1 \mid \rmX_0 = x_0$ is
\begin{equation*}
\frac{p(x_0,y)}{p_{\rmX_0}(x_0)} = \textrm{exp}\bigg({-\frac{1}{2(1-\beta^2)}(x_0^2 - 2\beta x_0y + y^2) + \frac{x_0^2}{2}}\bigg) .   
\end{equation*}
From the right-hand side, we see that $\rmX_1 \mid \rmX_0 = x_0$ is Gaussian with mean $\beta x_0$ and variance $1 - \beta^2$. Thus changing the noise level $\sigma$ alters both the mean and the variance of samples pushed forward via the Schr\"odinger bridge. In Proposition~6.1, one can check that $\beta = \beta(\alpha,\sigma)$ decreases from $1$ to $0$ as $\sigma \to \infty$, so for small $\sigma$ the bridge leaves samples close to their starting point, while for large $\sigma$ it exhibits stronger mean reversion toward the center of the distribution. At the same time, the conditional variance $1 - \beta^2$ increases from $0$ toward $1$; in particular, a Taylor expansion of $\beta$ at $\sigma = 0$ shows that $1 - \beta^2$ grows on the order of $\sigma^2$ for small $\sigma$. Consequently, larger values of $\sigma$ produce samples with greater spread. We expect that similar effects occur even when $\pi$ is not Gaussian, namely, that $\sigma$ should be directly related to the proximity of generated samples. Since our method provides an approximation of the mirror Schr\"odinger bridge, we expect the mirror bridge to inherit these properties.

\section{Practical Algorithm}
\label{sec:algorithm}

We now describe an algorithm to solve the mirror bridge 
problem numerically, based on our  scheme from (\ref{eq:direct-kl-2}). We choose our reference path measure $\sP^0 \in \sS$ to be the law of an Ornstein--Uhlenbeck process $\rmX_t$ given by the SDE 
$\rmd\rmX_t=-\alpha\rmX_t\rmd t+\sigma\rmd \rmW_t$, for some $\alpha>0$.

To compute the drift-averaging operator $\gA$, we need to be able to compute the backward drift $b_t^{\sP}$ of a path measure given its forward drift $f_t^{\sP}$, but in practice, we cannot simply apply (\ref{eq:nelson}) as we do not have access to the marginal densities $\rho_t^{\sP}$. We use the trajectory-caching method developed in \cite{bortoli2021diffusion, vargas2021likelihood} to estimate the backward drift $b_t^{\sP}$. Trajectory-caching is principled on the fact that $b_t^{\sP}$ can be expressed in terms of the expected rate of change in $\rmX_t$ over time. Concretely, we have the following formula, which can be taken as a formal definition of the backward drift of a diffusion process: 
\begin{equation} \label{eq:infinitesimal}
   b_{1-t}^{\sP}(x)=\lim_{\gamma\to0} \E\left[\frac{\rmX_{t-\gamma}-\rmX_{t}}{\gamma} \;\middle|\; \rmX_{t}=x\right].
\end{equation}
To apply (\ref{eq:infinitesimal}) in practice, take a positive integer $M$ and let $\{\gamma_i\}_{i=1}^M$ be a sequence of $M$ discrete time steps with sequence of partial sums $\{\bar{\gamma}_i\}_{i=1}^M$. Then we construct a discrete representation of the stochastic process $\rmX_t$ by using the Euler--Maruyama method to generate a collection of $N$ sample trajectories $\{\rmX_{i}^j\}_{i,j = 0}^{M-1,N-1}$ starting at the initial distribution $\pi$ in accordance with the SDE $\rmd\rmX_t=f_t^{\sP}(\rmX_t)\rmd t+\sigma\rmd \rmW_t$, where we know the forward drift $f_t^{\sP}$ as the output of the previous iteration of (\ref{eq:direct-kl-2}). Explicitly, for all $i \in \{0,\dots,M-2\}$ and $j \in \{0,\dots, N-1\}$, we have
\begin{equation} \label{eq:eulermaruyama}
\rmX_{i+1}^j= \rmX_{i}^j + f_{\bar{\gamma}_i}^{\sP}(\rmX_{i}^j)\gamma_i + \sigma^j\sqrt{\gamma_i} \rmZ_i^j,
\end{equation}
where $\rmZ_i^j \sim \gN(0,\mI)$.
The limiting quantity in (\ref{eq:infinitesimal}) is then leveraged as the target of the loss function used to train a neural network $v_t^{\theta}$, which approximates the backward drift $b_t^\sP$ for a specified range of $\sigma$ values $[\sigma_{\min},\sigma_{\max}]$. Specifically, we define the loss $\ell$ in terms of the optimization parameter $\theta$ by
\begin{equation} \label{eq:loss}
    \ell=\frac{1}{N} \sum_{i = 1}^{M - 1} \sum_{j = 0}^{N-1} \bigg|\bigg|v_{\bar{\gamma}_{i+1}}^\theta(\rmX_{i+1}^j) - \frac{\rmX_{i}^j - \rmX_{i+1}^j}{\gamma_{i+1}} -\left(f_{\bar{\gamma}_i}^{\sP}(\rmX_{i+1}^j) - f_{\bar{\gamma}_i}^{\sP}(\rmX_{i}^j)\right)\bigg|\bigg|^2 
\end{equation}
Observe that the first two terms in the loss constitute the difference between the drift and the infinitesimal rate of change of the process $\rmX_t$, i.e., the discretization of the difference between\begin{wrapfigure}{r}{0.5\linewidth}
\begin{minipage}{\linewidth}
    \vspace{-0.3in}
    \begin{algorithm}[H]
    \caption{\sc{Mirror bridge}}
    \label{alg:mirror}
    \begin{algorithmic}[1]
    \ampalgo
    \end{algorithmic}
    \end{algorithm}  
    \vspace{-0.3in}
\end{minipage} 
\end{wrapfigure} the left- and right-hand sides of (\ref{eq:infinitesimal}).
The network parameters $\theta$ are then learned via gradient descent with respect to the loss function $\ell(\theta)$. The resulting function $v_t^\theta$, where $\theta$ minimizes the loss $\ell(\theta)$, approximates the desired backward drift, as is suggested by \cite[Proposition 3]{bortoli2021diffusion}. We then average the forward and backward drifts of $\sP$ to obtain the next iteration of (\ref{eq:direct-kl-2}). Practically speaking, if the path measure iterates are denoted by $\sP^k$, then the forward and backward drifts of $\sP^k$ are parametrized by neural networks $v_t^{\theta^k}$ and $v_t^{\theta^{k+1/2}}$, and we take the forward drift of $\sP^{2k+1}$ to be the average output of the networks $v_t^{\theta^{k}}$ and $v_t^{\theta^{k+1/2}}$. This is summarized in Algorithm \ref{alg:mirror}, where we denote the limiting drift as $v_t^{\theta^\star}$ and the maximum number of drift-averaging iterations as $K$.

\section{Experiments}
\label{sec:experiments}
We demonstrate the utility of our method on a number of conditional resampling tasks, illustrating how it can be used to produce sample variations with control over the proximity to the initial sample.

\subsection{Gaussian transport}We start by comparing our method with two alternative algorithms, DSB \cite{bortoli2021diffusion} and DSBM \cite{shi2023diffusion}, when applied to the mirror Schr\"odinger bridge of multi-dimensional Gaussians. Fig. \ref{fig:convergence} shows that, in the case of dimension $d=50$, as the number of outer iterations increases, the empirical convergence of our method performs on par with both DSB and DSBM with the added benefit that each outer iteration with our algorithm requires half the training iterations. Recall that our method trains a single neural network to model a time-symmetrized drift function $v_t^\theta$ rather than a neural network for each of the forward and backward drift functions. More details on the derivation of the analytical solution, as well as information on parameters, can be found in the supplementary material. Additional results for dimensions $d=5,20$ can be found in the supplementary material.


\subsection{2D datasets}To illustrate the behavior of our method, we use our algorithm to resample from 2-dimensional distributions \cite{pedregosa2011scikitlearn}.  Unlike the mirror Schr\"odinger bridge with Gaussians, an analytical solution with these more general distributions is not known. We consider learning the drift function $v_t^{\theta}$ associated with the mirror bridge that flows samples from $\pdata$ to itself. The goal is to obtain new samples that are in the distribution $\pdata$ but exhibit some level of variation, i.e., in-distribution variation, correlated to the noise coefficient $\sigma$ in the diffusion process. Note that computing mirror bridges with a range of noise values by training one neural network is not possible using existing alternative methods.

\begin{figure}
    \begin{center}
    \includegraphics[width=\linewidth]{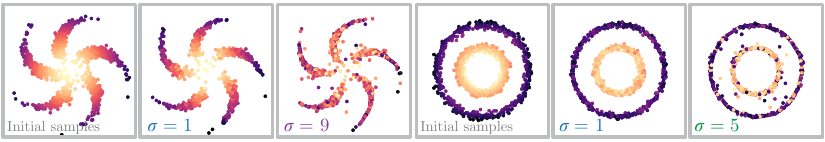}
    \end{center}
    \caption{Samples (color based on initial position) obtained using our method with various $\sigma$ values.}
    \label{fig:2d-experiment}
\end{figure}

In Fig. \ref{fig:2d-experiment}, we show the result of flowing samples via the mirror bridge with varying values of noise. We observe that the in-distribution variation of data points is controlled by $\sigma$, which can indeed be detected by the mixing of colors, or lack thereof, in each terminal distribution shown. For instance, on the right, we find mixing from samples between the inner and outer circles with the largest $\sigma$, compared with no mixing between circles with the smallest $\sigma$.

\subsection{Image resampling} \begin{wrapfigure}{l}{0.4\textwidth}
    \vspace{-0.25in}
    \begin{center}
    \includegraphics[width=\linewidth]{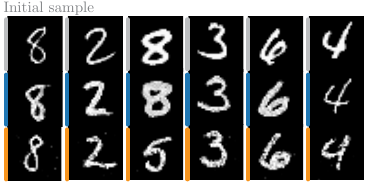}
    \end{center}
    \vspace{-0.1in}
    \caption{Samples produced by our mirror bridges for MNIST, using $\sigma=1,2$.}
    \label{fig:MNIST}
    \vspace{-0.2in}
\end{wrapfigure}We train our algorithm on the MNIST \cite{lecun1998mnist} and CelebA \cite{liu2015faceattributes} datasets. Training parameters and architecture for all experiments can be found in the supplementary material. Our results show that mirror bridges can be used to produce new samples from an image dataset with control over the proximity to the initial sample. In Fig. \ref{fig:MNIST}, we resample from MNIST using varying levels of noise (see also the supplementary material). Pushforward images obtained with a lower $\sigma$ value (\emph{middle}) are visually closer to the initial images (\emph{top}) than the ones obtained with a higher $\sigma$ value (\emph{bottom}).

Fig.~\ref{fig:closeup} shows the same control over in-distribution variation of pushforward samples using the RGB dataset CelebA. In each column, we exhibit a different sample from the dataset and, in each row, we show the corresponding pushforward obtained for different noise values. These results can be obtained without retraining the neural network. The typical metric to assess resampling quality for the image generation case is the Fréchet inception distance (FID) score, which we have plotted against training iterations. We observe FID scores decreasing with training iterations. The supplementary material includes more results using the CelebA dataset, including an analysis of the nearest neighbors in the dataset to the generated images. We find that the nearest neighbor of the generated sample is the initial sample itself, and the generated sample is distinct from all of its nearest neighbors, showing that our model does not simply regurgitate nearest neighbors of the initial sample as proximal outputs.

\begin{figure*}
    \centering
    \includegraphics[width=\linewidth]{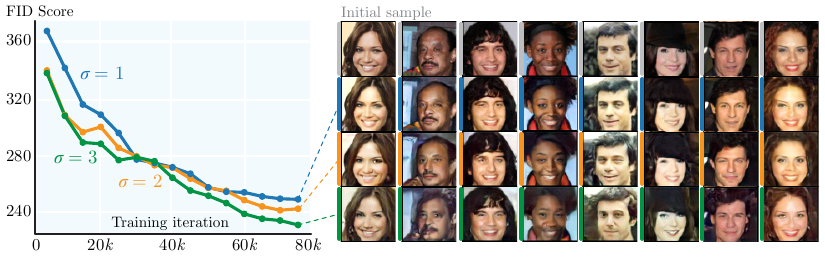}
    \vspace{-0.5cm}
    \caption{The control over in-distribution variance effect of $\sigma$ for a variety of initial samples (first row) from the empirical distribution of images in the CelebA dataset. 
    }
    \label{fig:closeup}
 \vspace{-0.2cm}
\end{figure*}

\subsection{Control Over Sample Proximity} We define proximity of samples using pixel-wise $L_2$ norm as our choice of distance metric. In Fig. \ref{fig:distance-regularity} (left), we demonstrate how larger values of $\sigma$ effectively produce pushforward samples that are farther in this distance metric, compared to samples generated with smaller values of $\sigma$. This experiment expands the results shown in Fig. \ref{fig:2d-experiment} to the case of resampling from image distributions. In particular, the mean and spread of the histograms in Fig. \ref{fig:distance-regularity} (right) show that larger values of $\sigma$ yield higher average distance values relative to the initial sample, as well as greater variation among these distances.

\begin{figure}
    \centering
    \includegraphics{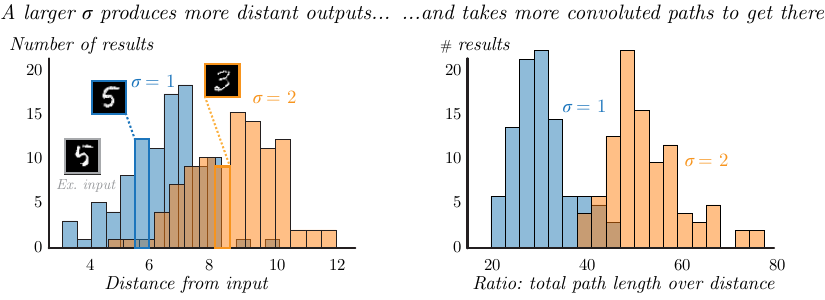}
    \vspace*{-0.2cm}
    \caption{On the left: Two histograms demonstrating how image samples generated with larger $\sigma$ correspond to less proximal samples relative to the initial image sample. On the right: Two histograms show the inverse ratio between displacement and total path length of sample paths as a metric of path straightness.}
    \label{fig:distance-regularity}
\vspace*{-0.2cm}
\end{figure}

\subsection{Sample Path Straightness} We present empirical results on the straightness of sample paths produced by our method. Specifically, in Fig. \ref{fig:distance-regularity} (right), we give a histogram for the values of a metric defined by taking the ratio of total displacement to total path length for different values of $\sigma$. For a given sample trajectory $\{\rmX_{i}\}_{i= 0}^{M-1}$, this metric is explicitly computed by dividing $\|\rmX_0-\rmX_{M-1}\|_2$ (total displacement) and $\sum_k\|\rmX_{k+1}-\rmX_{k}\|_2$ (total path length). The greater the value of this metric, the greater the variation in the trajectory; hence, smaller values of this metric are suggestive of greater sample path straightness. We find, as expected, that sample path straightness decreases as $\sigma$ increases.

\subsection{Comparison to Alternative Methods} \label{sec:comparison-mnist}
We compare our method with DSB and DSBM for image resampling with the MNIST dataset as the marginal distribution. For this experiment, we use the implementation for DSB and DSBM-IPF available in the \href{https://github.com/yuyang-shi/dsbm-pytorch/tree/main}{code repository} for \cite{shi2023diffusion}. We implement our algorithm based on the architecture provided, only modifying the model to take on $\sigma$ as an input for our method. We test all three methods with the same set of training parameters as described in the supplementary material. We train our model with $\sigma=1$ fixed to match the noise value in the SDE for the other two methods, which do not take $\sigma$ as a model input. 

We provide FID scores for each method in Fig. \ref{fig:comparison-mnist}. We observe that for DSB and DSBM, the forward and backward models result in pushforward samples of different quality. In particular, sample quality for the forward model is significantly lower than that of the backward. This indicates that neither of these methods converge to the mirror Schr\"odinger bridge for the given number of iterations, because the drift function for this bridge is necessarily time-symmetric, i.e., the forward and backward drifts must be equal to each other. In contrast, our algorithm provides time symmetry by construction: a single model is trained and forcibly ``symmetrized'' at each outer iteration via the drift-averaging procedure.

Also in Fig. \ref{fig:comparison-mnist}, we present a breakdown of runtime for each method obtained for the same experiment. Our method has significantly lower total runtime and average outer training iteration time. The latter is not surprising, considering that one of the key features of our algorithm is to eliminate training of a second neural network. We observe that the average inference time during training, however, is higher with our method. Overall, in this particular experiment, we see that our method makes a trade-off between a small reduction in sample quality for a significant speed-up in training, while also preserving time-symmetry.

\begin{figure}[t!]
    \centering
    \includegraphics[width=\linewidth]{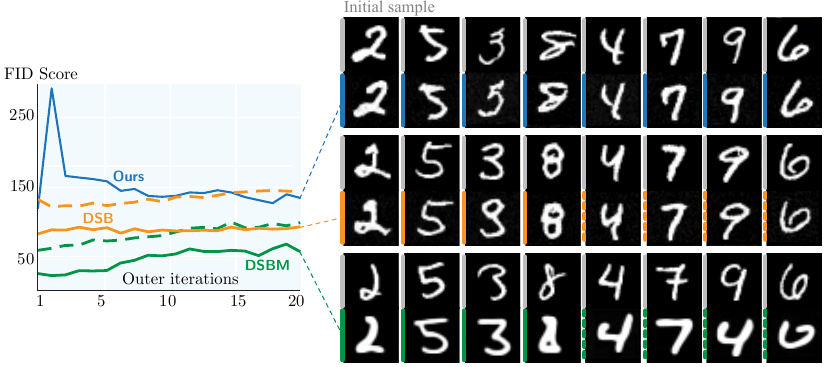}
    \vspace{-0.7cm}
    \caption{On the left: FID scores of pushforward samples versus outer iterations (single run) produced by our method, by DSB, and by DSBM, for a mirror bridge with the MNIST dataset as the marginal distribution. Solid lines correspond to backward models and dashed lines to forward models. On the right: Result of image resampling at outer iteration $20$. For each method and drift direction, the initial samples are displayed on the upper row and the pushforward samples on the lower.}
    \label{fig:comparison-mnist}
\end{figure}

\begin{table}[tbhp]
  \caption{Average runtimes for the experiment in Fig. \ref{fig:comparison-mnist}}
  \label{runtime-table}
  \centering
  \begin{tabular}{|c|c|c|c|c|}
    \hline
    & Ours     & DSB & DSBM-IPF &\\
    \hline
    Total          & $2.64$     & $5.25$    &  $12.47$ & \textit{hrs}\\
    Outer Iter.    & $7.94$     & $15.7$    &  $37.41$ & \textit{min}\\
    Inner Iter.    & $0.059$   & $0.055$   &  $0.209$ & \textit{s}\\
    Inference      & $2.009$    & $1.554$   &  $1.002$ & \textit{s}\\
    \hline
  \end{tabular}
\end{table}

\section{Conclusion}

By studying an overlooked version of the Schr\"odinger bridge problem, which we coin the \emph{mirror bridge}, we present an algorithm to sample with control over the in-distribution variation of new data points. Our method is flexible and requires fewer training iterations than existing alternatives \cite{bortoli2021diffusion,shi2023diffusion} designed for the general Schr\"odinger bridge problem. As a technique for estimating Schr\"{o}dinger bridges in the mirror case, our method presents an advantage over mirror interpolants \cite{albergo2023stochastic}, since optimizing the relative entropy of such an interpolant results in a $\min$-$\max$ optimization problem that is typically intractable \cite{shaul2023kinectic}. Our method is numerically tractable and well-principled, and it cuts down training in applications where control over in-distribution variation is desired. On the application front, we demonstrate that our method is a flexible tool to obtain new data points from empirical distributions in a variety of domains, including $2$-dimensional measures and image datasets. 

There are several promising avenues for future work. For instance, it would be valuable to sharpen the relationship between the fixed point of our drift-averaging iteration and the exact mirror Schr\"odinger bridge, thereby describing the accuracy of our estimate in a quantitative way. We also hope to study a potential $\sigma$ threshold for a sample to change class when resampled or to make ``class'' a neural network input, similar to text prompting in image generation. On the algorithmic side, an important next step would be to develop time-symmetrized variants of other methods for solving the Schr\"odinger bridge problem, like DSBM, which enjoys empirical advantages in efficiency and accuracy over its predecessor DSB.

\section*{Acknowledgments}
The authors would like to thank Vincent Sitzmann, Lingxiao Li, Artem Lukoianov and Christopher Scarvelis for discussion and feedback. We are also grateful to Suvrit Sra for thoughtful insights on related work. We thank Ahmed Mahmoud for proofreading.

\bibliographystyle{siamplain}
\bibliography{references}

\end{document}


\maketitle

\section{Proof of Lemma \ref{lem:h-to-eta-bounded}: Boundedness} \label{sec:bounded}

We now establish the boundedness of the map $h \mapsto \eta$. To do so, define the weighted energy
$$
E(t) \coloneqq \frac{1}{2}\int_{\R^n} \eta_t(x)^2 \rho_t^*(x)\,\rmd x.
$$
Differentiating $E(t)$ and substituting in the PDEs satisfied by $\eta_t$ and $\rho_t^*$ gives
\begin{align} \label{eq:Ederiv}
\frac{\rmd}{\rmd t} E(t)
&=  \int_{\R^n} \eta_t\Big(
    -\nabla\cdot(\eta_t v_t^*)
    -\nabla\cdot(\rho_t^* h_t)
    + \frac{\sigma^2}{2}\Delta\eta_t
  \Big)\rho_t^*\,\rmd x + \\
&\qquad\qquad \frac{1}{2}\int_{\R^n} \eta_t^2\Big(
    -\nabla\cdot(\rho_t^* v_t^*)
    + \frac{\sigma^2}{2}\Delta\rho_t^*
  \Big)\,\rmd x. \nonumber
\end{align}
We now integrate by parts in each term. Writing $\rho_t^* = \rho$ and $v_t^* = v$ for brevity, we obtain
\begin{align*}
-\int_{\R^n} \eta \rho \,\nabla\cdot(\eta v)\,\rmd x
&= \int_{\R^n} \nabla(\eta \rho)\cdot(\eta v)\,\rmd x
= \int_{\R^n} \big(\rho\nabla\eta + \eta\nabla\rho\big)\cdot(\eta v)\,\rmd x \\
&= \int_{\R^n} \rho\,\eta\,v\cdot\nabla\eta\,\rmd x
  + \int_{\R^n} \eta^2 v\cdot\nabla\rho\,\rmd x,
\end{align*}
\begin{align*}
-\int_{\R^n} \eta \rho \,\nabla\cdot(\rho h)\,\rmd x
&= \int_{\R^n} \nabla(\eta \rho)\cdot(\rho h)\,\rmd x
= \int_{\R^n} \big(\rho\nabla\eta + \eta\nabla\rho\big)\cdot(\rho h)\,\rmd x \\
&= \int_{\R^n} \rho^2 \nabla\eta\cdot h\,\rmd x
  + \int_{\R^n} \eta \rho h\cdot\nabla\rho\,\rmd x,
\end{align*}
\begin{align*}
\int_{\R^n} \eta \rho \,\Delta\eta\,\rmd x
&= -\int_{\R^n} \nabla(\eta \rho)\cdot\nabla\eta\,\rmd x
= -\int_{\R^n} \big(\rho\nabla\eta + \eta\nabla\rho\big)\cdot\nabla\eta\,\rmd x \\
&= -\int_{\R^n} \rho\,|\nabla\eta|^2\,\rmd x
  - \int_{\R^n} \eta\,\nabla\rho\cdot\nabla\eta\,\rmd x,
\end{align*}
\begin{align*}
-\int_{\R^n} \eta^2 \nabla\cdot(\rho v)\,\rmd x
&= \int_{\R^n} \nabla(\eta^2)\cdot(\rho v)\,\rmd x
= 2\int_{\R^n} \eta \rho v\cdot\nabla\eta\,\rmd x,
\end{align*}
\begin{align*}
\int_{\R^n} \eta^2 \Delta\rho\,\rmd x
&= -\int_{\R^n} \nabla(\eta^2)\cdot\nabla\rho\,\rmd x
= -2\int_{\R^n} \eta\,\nabla\eta\cdot\nabla\rho\,\rmd x.
\end{align*}
Substituting the five calculations above into (\ref{eq:Ederiv}), we find that
\begin{align}
\frac{\rmd}{\rmd t} E(t)
&= 2\int_{\R^n} \rho\,\eta\,v\cdot\nabla\eta\,\rmd x
 + \int_{\R^n} \eta^2 v\cdot\nabla\rho\,\rmd x
 + \int_{\R^n} \rho^2 \nabla\eta\cdot h\,\rmd x
 + \int_{\R^n} \eta \rho h\cdot\nabla\rho\,\rmd x \label{eq:Eprime-weighted}\\
&\qquad\qquad - \frac{\sigma^2}{2}\int_{\R^n} \rho\,|\nabla\eta|^2\,\rmd x
 - \sigma^2 \int_{\R^n} \eta\,\nabla\rho\cdot\nabla\eta\,\rmd x. \nonumber
\end{align}
Since $\rho>0$, we can write $\nabla\rho = \rho\nabla\log\rho$, and by assumption we have $\|\nabla\log\rho\|_{L^\infty}\leq B$. Using this, (\ref{eq:Eprime-weighted}) becomes
\begin{align}
\frac{\rmd}{\rmd t} E(t)
&= 2\int_{\R^n} \rho\,\eta\,v\cdot\nabla\eta\,\rmd x
 + \int_{\R^n} \eta^2 \rho\,v\cdot\nabla\log\rho\,\rmd x
 + \int_{\R^n} \rho^2 \nabla\eta\cdot h\,\rmd x
 + \int_{\R^n} \eta \rho^2 h\cdot\nabla\log\rho\,\rmd x \nonumber\\
&\qquad\qquad - \frac{\sigma^2}{2}\int_{\R^n} \rho\,|\nabla\eta|^2\,\rmd x
 - \sigma^2 \int_{\R^n} \eta\,\rho\nabla\log\rho\cdot\nabla\eta\,\rmd x. \label{eq:Eprime-weighted2}
\end{align}
We now estimate each term on the right-hand side of (\ref{eq:Eprime-weighted2}) in terms of the norms of $\eta_t$ and $h_t$. For any $\alpha>0$, by the AM-GM inequality we have
\begin{align*}
2\int_{\R^n} \rho\,\eta\,v\cdot\nabla\eta\,\rmd x
&\leq 2\|v^*\|_{L^\infty} \int_{\R^n} \rho\,|\eta||\nabla\eta|\,\rmd x \leq \|v^*\|_{L^\infty}\left(
    \alpha\int_{\R^n} \rho\,|\nabla\eta|^2\,\rmd x
  + \frac{1}{\alpha}\int_{\R^n} \eta^2 \rho\,\rmd x
\right),
\end{align*}
and, using the bound on $\nabla\log\rho$, we have
\begin{align*}
\left|\int_{\R^n} \eta^2 \rho\,v\cdot\nabla\log\rho\,\rmd x\right|
&\leq \|v^*\|_{L^\infty} \|\nabla\log\rho\|_{L^\infty} \int_{\R^n} \eta^2 \rho\,\rmd x
\leq C_1 \int_{\R^n} \eta^2 \rho\,\rmd x
\end{align*}
for a constant $C_1>0$ depending only on $\|v^*\|_{L^\infty}$ and $B$.

For the terms involving $h_t$, using $\rho\leq M$ together with AM-GM, for any $\beta>0$ we have
\begin{align*}
\int_{\R^n} \rho^2 \nabla\eta\cdot h\,\rmd x
&\leq \int_{\R^n} \rho^2 |\nabla\eta||h| \,\rmd x\leq M \int_{\R^n} \rho\,|\nabla\eta|\,\|h\|_2\,\rmd x \\
&\leq M\left(
    \frac{\beta}{2}\int_{\R^n} \rho\,|\nabla\eta|^2\,\rmd x
  + \frac{1}{2\beta}\int_{\R^n} \|h\|_2^2 \rho\,\rmd x
\right),
\end{align*}
and similarly
\begin{align*}
\int_{\R^n} \eta \rho^2 h\cdot\nabla\log\rho\,\rmd x
&\leq M \|\nabla\log\rho\|_{L^\infty} \int_{\R^n} \rho\,|\eta|\,\|h\|_2\,\rmd x \\
&\leq C_2\left(
    \frac{\gamma}{2}\int_{\R^n} \eta^2 \rho\,\rmd x
  + \frac{1}{2\gamma}\int_{\R^n} \|h\|_2^2 \rho\,\rmd x
\right)
\end{align*}
for any $\gamma>0$ and some constant $C_2>0$ depending only on $M$ and $B$.

Finally, for the last term in (\ref{eq:Eprime-weighted2}), for any $\delta>0$ we have
\begin{align*}
\left| \sigma^2 \int_{\R^n} \eta\,\rho\nabla\log\rho\cdot\nabla\eta\,\rmd x \right|
&\leq \sigma^2 \|\nabla\log\rho\|_{L^\infty} \int_{\R^n} \rho\,|\eta||\nabla\eta|\,\rmd x \\
&\leq \sigma^2 B\left(
    \frac{\delta}{2}\int_{\R^n} \rho\,|\nabla\eta|^2\,\rmd x
  + \frac{1}{2\delta}\int_{\R^n} \eta^2 \rho\,\rmd x
\right).
\end{align*}

Combining these bounds with the diffusion term $-(\sigma^2/2)\int \rho\,|\nabla\eta|^2\,\rmd x$ in (\ref{eq:Eprime-weighted2}), and choosing the parameters $\alpha,\beta,\delta>0$ sufficiently small, depending only on $\sigma$, $M$, $B$, and $\|v^*\|_{L^\infty}$, we can ensure that the terms involving the squared gradient $|\nabla\eta|^2$ have a nonpositive total coefficient and may therefore be dropped. Consequently, there exist constants $A,B'>0$, depending only on $M$, $B$, $\sigma$, and $\|v^*\|_{L^\infty}$, such that
$$
\frac{\rmd}{\rmd t} E(t)
\leq A E(t) + B' \int_{\R^n} \|h_t(x)\|_2^2 \rho_t^*(x)\,\rmd x.
$$
Since $\eta_0 = 0$, we have $E(0)=0$, and Gr\"onwall's inequality yields
\begin{equation} \label{eq:Et-bound-weighted}
E(t)
\leq C \int_0^t \int_{\R^n} \|h_s(x)\|_2^2 \rho_s^*(x)\,\rmd x\,\rmd s
\end{equation}
for every $t \in [0,1]$, where $C>0$ depends only on $A,B'$. Integrating (\ref{eq:Et-bound-weighted}) in time over $[0,1]$ and using $t\leq 1$ gives
$$
\int_0^1 E(t)\,\rmd t
\leq C \int_0^1 \int_{\R^n} \|h_s(x)\|_2^2 \rho_s^*(x)\,\rmd x\,\rmd s.
$$
By the definition of $E(t)$, this is exactly  $\|\eta\|_2 \leq C \| h\|_2$, as necessary.

\section{Proof of Proposition \ref{prop:analytical}: MSB in the Gaussian Case} \label{sec:analyticalproof}

We follow the proof of Proposition 46 of \cite{bortoli2021diffusion}, which established the corresponding result in the case where the reference process has zero drift. Imitating the proof of Proposition 43 of \cite{bortoli2021diffusion}, we see that the static Schr\"{o}dinger bridge $\pi^*$ exists and is a $2d$-dimensional Gaussian. That the mean equals zero follows from the fact that both marginals have zero mean. The rest of the proof is devoted to determining the covariance matrix $\Sigma$ of $\pi^*$.

The fact that marginals have unit variance implies that $\Sigma_{00} = \Sigma_{11} = \mI$. To compute $\Sigma_{01}$ and $\Sigma_{10}$, we start by computing the probability density function (PDF) $p^0(x,y)$ of the reference measure $\pi^0$, where $x,y \in \R^d$. Recall that $p^0(x,y)$ is the product of the conditional PDF of $\rmX_1 \mid \rmX_0$ with the PDF of $\rmX_0$. Thus, we have
$$p^0(x,y) = p_{\rmX_1 \mid \rmX_0}(x,y) \times p_{\rmX_0}(x).$$
Note that $\rmX_0$ has zero mean and unit variance, so up to normalization we have
$$p_{\rmX_0}(x) \propto e^{-\frac{x^2}{2}}.$$
On the other hand, the mean and variance of the conditional distribution $\rmX_1 \mid \rmX_0$ are computed in section II of \cite{trajanovski2023OU}, where it is shown that they are respectively given by $$xe^{-\alpha} \quad \text{and} \quad \sigma_1^2 \coloneqq \frac{\sigma^2}{2\alpha}(1 - e^{-2\alpha}).$$
It follows that
$$p_{\rmX_1 \mid \rmX_0}(x,y) \propto e^{-\frac{1}{2\sigma_1^2}\left(y - e^{-\alpha}x\right)^2}.$$
Combining these calculations, we conclude that the joint distribution has PDF given by 
$$p^0(x,y) \propto e^{-\frac{1}{2}\left((1 + \sigma_1^{-2}e^{-2\alpha})x^2 - 2\sigma_1^{-2}e^{-\alpha}xy + \sigma_1^{-2}y^2 \right)}.$$
This distribution is evidently a Gaussian with zero mean and covariance matrix $\Sigma^0$ given by
$$\Sigma^0 = \begin{pmatrix} \mI & e^{-\alpha} \mI \\ e^{-\alpha}\mI & (\sigma_1^2+e^{-2\alpha}) \mI \end{pmatrix}.$$
Note in particular that the variance of the marginal of $\pi^0$ at $t = 1$ is equal to the coefficient of the bottom-right entry of $\Sigma^0$, which is $\sigma_1^2 + e^{-2\alpha}$. Now, the KL divergence between a $2$-dimensional Gaussian distribution $\widetilde{\pi}$ with zero mean and covariance matrix $\widetilde{\Sigma}$ and the distribution $\pi^0$ is given explicitly by
$$D_{\rm{KL}}(\widetilde{\pi} \;\Vert\; \pi^0) = \frac{1}{2}\left(\log\frac{\det \Sigma^0}{\det \widetilde{\Sigma}} - d + \Tr\big({\Sigma^0}^{-1}\widetilde{\Sigma}\big)\right).$$
If we take $\widetilde{\Sigma}$ to be of the form
$$\widetilde{\Sigma} = \begin{pmatrix} \mI & S \\ S^T & \mI \end{pmatrix},$$
which matches the form of the covariance $\Sigma$ for $\pi^*$, then 
$$D_{\rm{KL}}(\widetilde{\pi} \;\Vert\; \pi^0) = \frac{1}{2}\left(-\log \det \widetilde{\Sigma} - 2e^{-\alpha }\sigma_1^{-2} \Tr(S)  + C\right)$$
where $C \in \R$ is a nonzero constant independent of $\widetilde{\Sigma}$. As argued in the proof of Proposition 46 of \cite{bortoli2021diffusion}, we can assume $S = S^T$ is a symmetric matrix, as doing so will only decrease $D_{\rm{KL}}(\widetilde{\pi} \;\Vert\; \pi^0)$, so $S$ is diagonalizable. Let $\lambda_1, \dots, \lambda_d$ denote the eigenvalues of $S$, counted with multiplicity. Using the well-known formula for the determinant of a block $2 \times 2$ matrix, we find that
$$\det \widetilde{\Sigma} = \det(\mI - S^2) = \prod_{i = 1}^d (1 - \lambda_i^2).$$
Thus, we obtain
$$D_{\rm{KL}}(\widetilde{\pi} \;\Vert\; \pi^0) = \frac{1}{2}\sum_{i = 1}^d f(\lambda_i) + C, \quad \text{where} \quad f(x) = -\log(1 - x^2) - 2e^{-\alpha}\sigma_1^{-2}x.$$
Note in particular that since $\widetilde{\Sigma}$ is a covariance matrix, it is positive semi-definite, and so its eigenvalues $1 - \lambda_i^2$ must be nonnegative, implying that $|\lambda_i| \leq 1$ for each $i$.

Minimizing $\KL(\widetilde{\pi} \;\Vert\; \pi^0)$ then amounts to take $\lambda_1 = \cdots = \lambda_d = \beta$ in such a way that $f(\beta)$ is minimized. Observe that the equation 
$$f'(\beta) = \frac{2\beta}{1-\beta^2} - 2e^{-\alpha}\sigma_1^{-2} = 0$$
is solved by
$$\beta = \frac{\sigma^2(1-e^{2 \alpha }) \pm \sqrt{16 e^{2 \alpha } \alpha ^2+\sigma^4\left(1-e^{2 \alpha }\right)^2}}{4 \alpha e^\alpha}.$$
We then choose the sign to be $+$ to ensure that $|\beta| \leq 1$.

\section{Proof of Proposition~\ref{prop:gausserror}: Drift-Averaging for Gaussian Transport} \label{sec:gausserror}

Start by defining the ratio
$$r(x) \coloneqq \frac{\pi(x)}{\rho_0(x)}
  = \sqrt{\frac{q_0}{q_1}}\, \exp(ax^2),
  \quad \text{where} \quad
  a \coloneqq \frac12\left(\frac{1}{q_0} - \frac{1}{q_1}\right).$$
For the initial constraint, the direct KL projection onto $\sD(\pi,\cdot)$ has Radon--Nikodym derivative
$$\frac{\rmd\sP'}{\rmd\sP^0}(\rmX_\cdot)
  = \frac{1}{Z'} r(\rmX_0),$$
for a normalizing constant $Z'$. By Proposition~\ref{prop:directmarginal}, we have $f_t^{\sP'}(x) = f_t^{\sP^0}(x) = -\alpha x$. As for the final constraint, the direct KL projection onto $\sD(\cdot,\pi)$ has Radon--Nikodym derivative
$$\frac{\rmd\sP''}{\rmd\sP^0}(\rmX_\cdot)
  = \frac{1}{Z''} r(\rmX_1),$$
for a normalizing constant $Z''$. For each $t \in [0,1]$, define the Schr\"odinger potential
$$\beta_t(x)
  \coloneqq \mathbb{E}_{\sP^0}\bigl[r(\rmX_1)\mid \rmX_t = x\bigr].$$
By Nelson’s relation (\ref{eq:nelson}), the forward drift of $\sP''$ is given by
\begin{equation}\label{eq:nelsonapplied}
  f_t^{\sP''}(x)
  = f_t^{\sP^0}(x) + \sigma^2\,\partial_x\log \beta_t(x)
  = -\alpha x + \sigma^2\,\partial_x\log \beta_t(x).
\end{equation}

Under the reference $\sP^0$, the pair $(\rmX_1,\rmX_t)$ is jointly Gaussian with mean $(0,0)$ and covariance
$$\mathrm{Cov}(\rmX_1,\rmX_1) = q_0 =  \mathrm{Cov}(\rmX_t,\rmX_t),\quad
  \mathrm{Cov}(\rmX_1,\rmX_t) = q_0c_t,
  \quad \text{where} \quad
  c_t \coloneqq e^{-\alpha(1-t)}.$$
Therefore, conditional on $\rmX_t=x$, we have that $\rmX_1$ is Gaussian with mean and variance
$$\mu_t(x) \coloneqq \mathbb{E}[\rmX_1\mid \rmX_t=x] = c_t x,
  \qquad
  s_t^2 \coloneqq \mathrm{Var}(\rmX_1\mid \rmX_t=x) = q_0\bigl(1-c_t^2\bigr).$$
Let $Y$ be a scalar Gaussian $Y\sim\mathcal{N}(\mu,s^2)$. For any real $a$ with $1-2a s^2>0$, one has
\begin{equation}\label{eq:Gaussian-moment}
  \mathbb{E}\bigl[\exp(aY^2)\bigr]
  = \frac{1}{\sqrt{1-2a s^2}}\,
    \exp \left(
      \frac{a\mu^2}{1-2a s^2}
    \right).
\end{equation}
Since $r(x) \propto \exp(ax^2)$, multiplicative constants factor out of the conditional expectation defining $\beta_t$. Applying (\ref{eq:Gaussian-moment}) with $Y=(\rmX_1\mid \rmX_t=x)$, $\mu=\mu_t(x)$, and $s^2=s_t^2$, we obtain
$$\beta_t(x)
  = \frac{1}{\sqrt{1-2a s_t^2}}
    \exp \left( \frac{ac_t^2}{1-2a s_t^2}x^2 \right).$$
Substituting this into~\eqref{eq:nelsonapplied} gives
\begin{equation}\label{eq:fB}
  f_t^{\sP''}(x)
  = -\alpha x + \sigma^2\cdot 2\frac{ac_t^2}{1-2a s_t^2}x
  = \left(-\alpha + 2\sigma^2\,\frac{ac_t^2}{1-2a s_t^2}\right)x.
\end{equation}

Kurras' symmetrization step is defined by taking the geometric mean of the two one-sided projections \cite{kurras2015symmetric}. Concretely, we define a new path measure $\sP^{\mathrm{sym}}$ by
$$\frac{\mathrm{d}\sP^{\mathrm{sym}}}{\mathrm{d}\sP^0}(\rmX_\cdot)
  \propto \sqrt{ \frac{\mathrm{d}\sP'}{\mathrm{d}\sP^0}(\rmX_\cdot)\,
                   \frac{\mathrm{d}\sP''}{\mathrm{d}\sP^0}(\rmX_\cdot) }
  \propto \sqrt{r(\rmX_0)\,r(\rmX_1)}.$$
Since $r(x) \propto \exp(ax^2)$, we may write
$$\frac{\mathrm{d}\sP^{\mathrm{sym}}}{\mathrm{d}\sP^0}(\rmX_\cdot)
  \propto h(\rmX_0)\,h(\rmX_1),
  \quad \text{where} \quad
  h(x) \coloneqq \exp \left(\frac{a}{2}\,x^2\right).$$
As before, the forward drift of $\sP^{\mathrm{sym}}$ is given by
$$f^{\mathrm{sym}}_t(x)
  = f_t^{\sP^0}(x) + \sigma^2\partial_x\log h_t(x)
  = -\alpha x + \sigma^2\partial_x\log h_t(x),$$
where
$$h_t(x) \coloneqq \mathbb{E}_{\sP^0}\bigl[h(\rmX_1)\mid \rmX_t=x\bigr]
  = \mathbb{E}_{\sP^0}\left[ \exp \left(\frac{a}{2}\rmX_1^2\right)\,\middle|\,\rmX_t=x \right].$$
We again use (\ref{eq:Gaussian-moment}), now with $a$ replaced by $a/2$ to conclude that
$$h_t(x) = \frac{1}{\sqrt{1-a s_t^2}}\exp \left( \frac{(a/2)c_t^2}{1-a s_t^2}\,x^2 \right).$$
Thus the forward drift of the symmetric IPFP step is
\begin{equation}\label{eq:fsym}
  f^{\mathrm{sym}}_t(x)
  = -\alpha x + \sigma^2\,\frac{ac_t^2}{1-a s_t^2}\,x
  = \left(-\alpha + \sigma^2\,\frac{ac_t^2}{1-a s_t^2}\right)x.
\end{equation}
On the other hand, our drift-averaging approximation takes the arithmetic mean of the forward drifts of $\sP'$ and $\sP''$:
$$f^{\mathrm{avg}}_t(x)
  \coloneqq \frac12\Bigl(f_t^{\sP'}(x) + f_t^{\sP''}(x)\Bigr).$$
Using $f_t^{\sP'}(x) = -\alpha x$ and (\ref{eq:fB}), we obtain
\begin{equation}\label{eq:favg}
  f^{\mathrm{avg}}_t(x)
  = \left(-\alpha + \sigma^2\,\frac{ac_t^2}{1-2a s_t^2}\right)x.
\end{equation}

We now compare (\ref{eq:fsym}) and (\ref{eq:favg}) when the target marginal $\pi$ is close to the stationary marginal. Let $q_1 = q_0(1+\varepsilon)$ for some small parameter $\varepsilon$, with $0 \leq |\varepsilon|\ll 1$. Then
$$\frac{1}{q_1} = \frac{1}{q_0(1+\varepsilon)}
  = \frac{1}{q_0}\bigl(1 - \varepsilon + \varepsilon^2 + O(\varepsilon^3)\bigr),$$
and therefore
\begin{align*}
  a & = \frac12\left(\frac{1}{q_0} - \frac{1}{q_1}\right)
   = \frac12\frac{1}{q_0}\Bigl(1 - (1 - \varepsilon + \varepsilon^2 + O(\varepsilon^3))\Bigr) = \frac{\varepsilon}{2q_0} - \frac{\varepsilon^2}{2q_0} + O(\varepsilon^3).
\end{align*}
In particular, $a = O(\varepsilon)$ as $\varepsilon\to 0$. For fixed $t\in[0,1]$, $c_t$ and $s_t^2$ are constants, so we can treat them as such when expanding in powers of $a$. Define
$$C^{\mathrm{sym}}_t \coloneqq \frac{ac_t^2}{1-a s_t^2},
  \qquad
  C^{\mathrm{avg}}_t \coloneqq \frac{ac_t^2}{1-2a s_t^2},$$
so that
$$f^{\mathrm{sym}}_t(x) = \bigl(-\alpha + \sigma^2 C^{\mathrm{sym}}_t\bigr)x,
  \qquad
  f^{\mathrm{avg}}_t(x) = \bigl(-\alpha + \sigma^2 C^{\mathrm{avg}}_t\bigr)x.$$
For $|a|$ small,
$$\frac{1}{1-a s_t^2} = 1 + a s_t^2 + a^2 s_t^4 + O(a^3),
  \qquad
  \frac{1}{1-2a s_t^2} = 1 + 2a s_t^2 + 4a^2 s_t^4 + O(a^3).$$
Therefore, $C^{\mathrm{sym}}_t - C^{\mathrm{avg}}_t
  = -a^2c_t^2 s_t^2 + O(a^3)$, and so
$$f^{\mathrm{sym}}_t(x) - f^{\mathrm{avg}}_t(x)
  = \sigma^2\bigl(C^{\mathrm{sym}}_t - C^{\mathrm{avg}}_t\bigr)x
  = -\sigma^2 a^2c_t^2 s_t^2x + O(a^3)x.$$
Since $a = O(\varepsilon)$, we conclude $f^{\mathrm{sym}}_t(x) - f^{\mathrm{avg}}_t(x) = O(\varepsilon^2)x$ uniformly in $t\in[0,1]$ as $\varepsilon\to 0$.
  
\section{Implementation Details and Additional Results}
\label{sec:details}

In this section we give further details on our experimental setup and include some additional results. Akin to \cite[Technique 5]{song2020techniques}  and \cite[Technique 6]{bortoli2021diffusion}, we improve performance of Algorithm \ref{alg:mirror} by implementing the exponential moving average (EMA) of network parameters. We run our experiments on a NVIDIA GeForce RTX $3090$ GPU with $24$GB of memory.

\subsection{Gaussian Transport}

We use the MLP large network from \cite{bortoli2021diffusion} for DSB and DSBM in all Gaussian transport experiments. For our method, we modify this network to take $\sigma$ as an input. The values of $\sigma$ are uniformly sampled from the (inclusive) interval from $1$ to $5$ for training, and at test time we fix $\sigma=1$ for all samples to compare with DSB and DSBM, which do not take $\sigma$ as a network input, but each use $\sigma=1$ via the SDE discretization.  We run the same experiment for dimension $d=5$ and $d=20$ (in Fig. \ref{fig:additional-convergence}), and $d=50$ (in Fig. \ref{fig:convergence}). The number of samples for all experiments is $10{,}000$. We use $20$ timesteps and train for $10{,}000$ inner iterations for each of $20$ outer iterations.

\begin{figure}
    \centering
    \includegraphics[width=\linewidth]{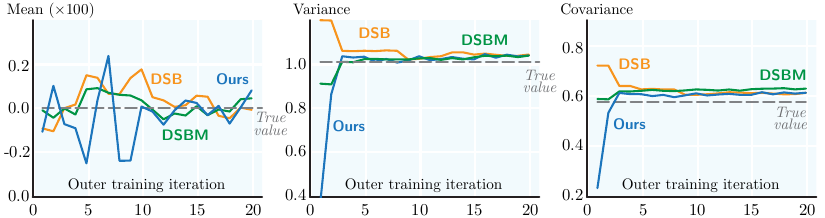}
    \includegraphics[width=\linewidth]{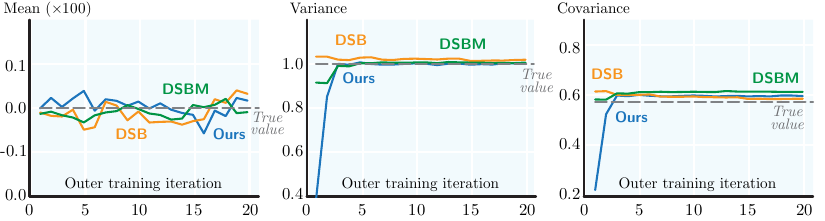}
    \caption{For each method, we plot the mean (left) and variance (middle) obtained for the terminal samples, i.e. samples obtained at time $t=T$, as well as the covariance (right) of the joint distribution, versus the number of outer iterations, averaged over 5 trials. Top: $d=5$. Bottom: $d=20$.}
    \label{fig:additional-convergence}
\end{figure}

\subsection{2D Datasets}

We modify the network architecture with positional encoding from \cite{vaswani2017attention}, which is used by \cite{bortoli2021diffusion}, to take values of noise $\sigma$ rather than tuples of only $\rmX$ and $t$. The values of $\sigma$ are concatenated to the spatial features before the first MLP block is applied. This modified network is used to parametrize our drift function. We use Adam optimizer with learning rate $10^{-4}$ and momentum $0.9$. We train each example for $10{,}000$ inner iterations per outer iteration of the algorithm. Fig. \ref{fig:2d-experiment} shows the terminal samples obtained for outer iteration $30$ for all example datasets. The noise values $\sigma^j$ are sampled uniformly in the range from $1$ to $9$ for training. At test time, a fixed $\sigma$ value is chosen for all sample trajectories. We train with $10{,}000$ samples, which are refreshed each $1{,}000$ iterations. We use $20$ timesteps of size $0.01$ each. All $2$-dimensional experiments run on CPU.

We compute Chamfer distances as a means of measuring the proximity of the pushforward distributions exhibited in Fig. \ref{fig:2d-experiment} to the corresponding initial distributions. In the mirror case, the pushforward distribution should match the initial distribution, and the Chamfer distance between them should therefore decay as the number of iterations grows. In Fig. \ref{fig:chamfer}, we demonstrate how the Chamfer distance decays over outer iterations of our method for the same 2D distribution with different values of $\sigma$ (left), as well as how the Chamfer distance decays for different datasets with fixed $\sigma$ value (right).

\begin{figure}
    \centering
    \includegraphics{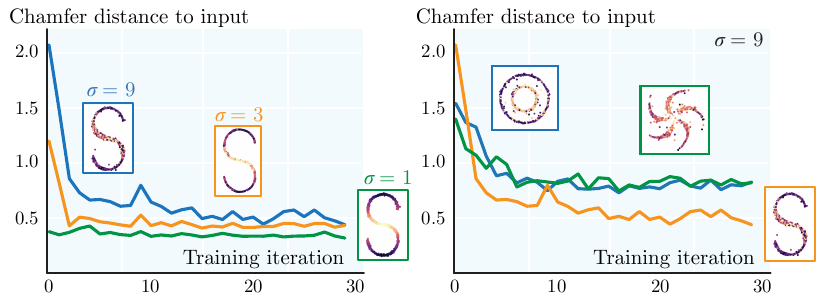}
    \caption{On the left: Three curves, each corresponding to a different $\sigma$ value, showing convergence using Chamfer distance for the same 2D dataset (shown in Fig. \ref{fig:2d-experiment}). On the right: Three curves, each corresponding to a different 2D dataset, showing convergence for a fixed $\sigma$ value.}
    \label{fig:chamfer}
\end{figure}

\subsection{Image Resampling}\label{sec:image-details}

For the image dataset experiments, we modify the U-Net architecture used in \cite{bortoli2021diffusion, shi2023diffusion} to take values of noise $\sigma$. Each value $\sigma^j$ is expanded to match image size and concatenated to channels of their corresponding sample image $j$ before the input block is applied. For all image experiments we follow the timestep $\gamma$ schedule used in \cite{bortoli2021diffusion} with $\gamma_{\rm{min}}=10^{-5}$ and $\gamma_{\rm{max}}=0.1$.  We use Adam optimizer with learning rate $10^{-4}$ and momentum $0.9$. Experiments with image datasets were run on limited shared GPU resources; lower-resolution image sizes and number of samples in cache were chosen accordingly.

\subsection*{MNIST} For the experiment in Fig. \ref{fig:MNIST}, we use $10{,}000$ cached images of size $28\times28$; the batch size is $128$ and the number of timesteps is $30$. The noise values are sampled uniformly in the interval from $1$ to $5$ (inclusive) during training. We train for $5{,}000$ iterations per outer iterations, and cached samples are refreshed every $1{,}000$ inner iterations. The terminal samples shown are for outer iteration $8$.

\begin{figure}
    \centering
    \includegraphics[width=\linewidth]{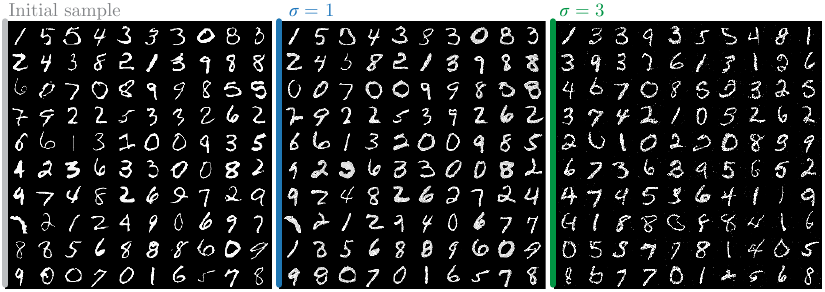}
    \caption{Additional results for the empirical distribution of handwritten digits from which the examples in Fig. \ref{fig:MNIST} are obtained.}
    \label{fig:mnist-additional}
\end{figure}

\subsection*{CelebA} In Fig. \ref{fig:closeup} and \ref{fig:CelebA}, we use $300$ cached images of size $64\times64$ and batch size $128$. The cache is refreshed every $100$ inner iterations and we train for $5{,}000$ iterations per outer iterations. The number of timesteps is $50$; the $\sigma$ values are uniformly sampled in the interval from $1$ to $3$. The terminal sample images are shown for outer iteration $15$. The FID score in Fig. \ref{fig:closeup} is computed using $300$ images.

\begin{figure}
    \centering
    \includegraphics{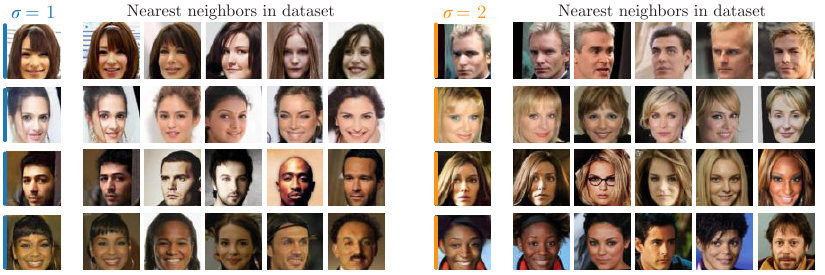}
    \caption{For our generated results (first and seventh columns), we show the five nearest neighbors in the CelebA dataset as measured through the features extracted by ResNet50 \cite{he2016deep}.}
    \label{fig:knn}
\end{figure}

\begin{figure}[ht]
    \centering
    \includegraphics[width=\linewidth]{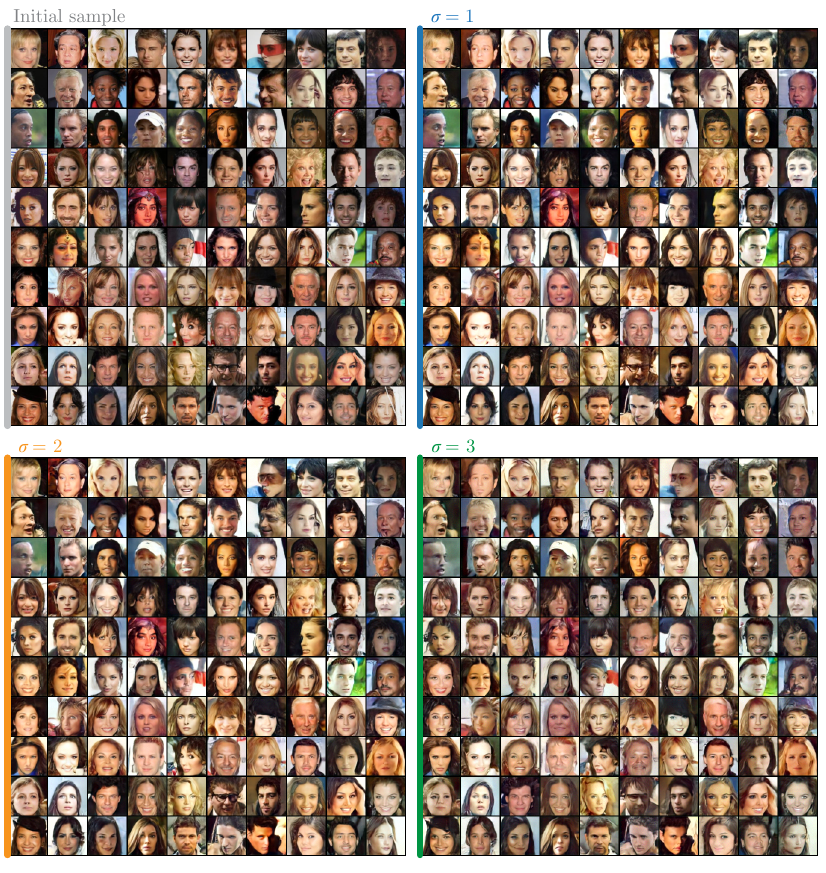}
    \caption{Additional results for the empirical distribution of images in CelebA from which the examples in Fig. \ref{fig:closeup} are obtained.}
    \label{fig:CelebA}
\end{figure}

\bibliographystyle{siamplain}
\bibliography{references}
